\documentclass{article}

\usepackage[utf8]{inputenc} 
\usepackage[T1]{fontenc}    

\usepackage[pagebackref]{hyperref} 
\usepackage{url}            

\usepackage[dvipsnames,table]{xcolor} 

\usepackage{graphicx}       
\usepackage{booktabs}       
\usepackage{multicol}       
\usepackage{multirow}       
\usepackage{makecell}       
\usepackage{colortbl}       
\usepackage{caption}        
\usepackage{subcaption}     
\usepackage{wrapfig}        
\usepackage{adjustbox}      
\usepackage{titletoc}       

\usepackage{mathtools}
\usepackage{amsmath,amsthm} 
\usepackage{amsfonts}       
\usepackage{amssymb}        
\usepackage{nicefrac}       
\usepackage{microtype}      

\usepackage{algorithm,algorithmicx,algpseudocode}

\usepackage{comment}        
\usepackage{array}

\usepackage{doi}            
\usepackage[normalem]{ulem} 
\usepackage{placeins}       
\usepackage[detect-all]{siunitx} 

\usepackage{bbding}         
\definecolor{cold}{HTML}{008acd}
\def\frozen{\raisebox{-.15ex}{\color{cold}\SnowflakeChevron}}
\DeclareUnicodeCharacter{2744}{\frozen}
\DeclareUnicodeCharacter{1F525}{\unfrozen}

\usepackage{pifont}         

\definecolor{vlightgray}{gray}{0.87}

\definecolor{tblue}{HTML}{1F77B4}
\definecolor{torange}{HTML}{FF7F0E}
\definecolor{tgreen}{HTML}{2CA02C}
\definecolor{tred}{HTML}{FF0000}

\definecolor{linkcolor}{HTML}{991408}  
\definecolor{citecolor}{HTML}{2E7E2A}  
\definecolor{filecolor}{HTML}{131877}  
\definecolor{menucolor}{HTML}{727500}  
\definecolor{runcolor} {HTML}{137776}  
\definecolor{urlcolor} {HTML}{0a2bbf}  
\hypersetup{colorlinks=true,linkcolor=linkcolor,citecolor=citecolor,filecolor=filecolor,menucolor=menucolor,runcolor=runcolor,urlcolor=urlcolor}
\newcommand{\appref}[1]{\hyperref[#1]{Appendix~\ref*{#1}}}

\def\Snospace~{\S{}}

\newcommand{\hider}[1]{}

\newtheorem{theorem}{Theorem}
\theoremstyle{definition}
\newtheorem{definition}{Definition}
\usepackage{tikz}
\usetikzlibrary{positioning,arrows.meta,shapes.geometric,decorations.pathreplacing}

\tikzset{
  maincircle/.style={circle, draw, minimum size=13mm, inner sep=2pt, font=\footnotesize, align=center},
  fwd/.style={->, thick, >=Latex},
  promise/.style={->, dashed, thick, >=Latex,
                  preaction={draw, white, line width=3pt}}, 
  diamondnode/.style={diamond, draw=green!60, fill=green!5, very thick, minimum size=7mm},
  node/.style={circle, draw=green!60, fill=green!5, very thick, minimum size=7mm},
  orange/.style={draw=orange!60, fill=orange!5},
  red/.style={draw=red!60, fill=red!5},
  blue/.style={draw=blue!60, fill=blue!5},
  teal/.style={draw=teal!60, fill=teal!5},
  purple/.style={draw=purple!60, fill=purple!5},
  highlight/.style={ultra thick},
}

\def\1{\bm{1}}

\DeclareMathAlphabet{\mathsfit}{\encodingdefault}{\sfdefault}{m}{sl}
\SetMathAlphabet{\mathsfit}{bold}{\encodingdefault}{\sfdefault}{bx}{n}

\usepackage{iclr2026/iclr2026_conference}
\usepackage{times} 

\usepackage{natbib}
\usepackage{etoolbox}       

\newtoggle{arxiv}
\togglefalse{arxiv}

\title{Always Keep Your Promises: A Model-Agnostic Attribution Algorithm for Neural Networks}

\author{
{\bf 
Kevin Lee$^{\sharp}$
\,\
Duncan Halverson$^{\S}$
\,\
Pablo Millan~Arias$^{\sharp, \S}$}
\\
$^\sharp$ David R. Cheriton School of Computer Science, University of Waterloo, ON, Canada  \\
$^\S$ Scotiabank AML AI Research, Toronto, ON, Canada \\ 
{\footnotesize 
   \texttt{\{\href{mailto:k327lee@uwaterloo.ca}{k327lee}, \href{mailto:pmillana@uwaterloo.ca}{pmillana}\}@uwaterloo.ca}
   } \\
{\footnotesize
   \texttt{
   \href{mailto:duncan.smith-halverson@scotiabank.com}{duncan.smith-halverson@scotiabank.com}}
   } 
}

\iclrfinalcopy
\begin{document}

\maketitle

\begin{abstract}
Inability to precisely understand neural network outputs is one of the most severe issues limiting the use of AI in multiple domains, from science and medicine to high-stakes decision and regulatory models. Layer-wise Relevance Propagation (LRP) is an established explainability method that addresses some of these limitations, but widespread adoption has not been possible because existing implementations must be coupled with individual model architectures, rendering them impractical for a rapidly evolving model space.
Our algorithm, DynamicLRP, is a lightweight and flexible method for performing LRP on any neural network with provably minimal overhead. To achieve this, we introduce a novel graph search mechanism called the ``Promise System'' which repurposes deep learning computation graphs for non-gradient computations and is implemented at the primitive tensor operation level using standard automatic differentiation libraries.
We demonstrate that DynamicLRP matches or surpasses specialised implementations in attribution quality across vision and language tasks and remains efficient for models at the billion-parameter scale.
Notably, our implementation achieved 99.99\% operation coverage across 31,465 computation nodes from 15 diverse architectures, without any architecture-specific modifications.
This is the first truly model-agnostic LRP solution, enabling high-quality neural network attribution across the full spectrum of modern AI architectures. All code is available at \url{https://github.com/keeinlev/dynamicLRP}.
\end{abstract}

\section{Introduction}

Modern neural networks, in particular massive transformer-based foundation models \cite{gemmateam2025gemma3, dosovitskiy2021vit, dubey2024llama3}, offer unprecedented ability to map complex, unstructured inputs to accurate outputs. 
Inability to understand this mapping severely limits use of these powerful models in all domains where the reverse mapping is wanted. Pertinent domains include, but are not limited to: natural science, social science, medicine, law, and finance. 

While a variety of methods for Explainable AI (XAI) do exist, each has substantive drawbacks.
The most well-known methods include perturbation-based approaches such as SHAP \cite{lundberg2017unified} and LIME \cite{ribeiro2016should}. 
These methods are thoroughly unsuitable for use in massive models, with SHAP explanations becoming more misleading with input cardinality \cite{huang2024failings}, and LIME suffering from an arbitrary definition of sensitivity to locality. 
These methods are further troubled by the exponential computation incurred proportional to model complexity \cite{covert2020understanding}. 
Gradient-based methods \cite{simonyan2014deep, sundararajan2017axiomatic, smilkov2017smoothgrad} are fragile, suffering from saturation, shattering, and assumptions of linearity \cite{ancona2018towards, samek2021xaireview}.
Class Activation Mapping \cite{zhou2016learning, selvaraju2017grad} lacks granular precision. 
In contrast, Layer-wise Relevance Propagation (LRP), a backpropagation-based method, \cite{bach2015pixel} provides a principled, efficient alternative by decomposing predictions into input contributions via strict conservation properties \cite{montavon2017explaining}, avoiding the pitfalls of local sensitivity \cite{samek2021xaireview} and proving effective in diverse domains \cite{iwana2019lrpexplainingcnns, bohle2019alzheimers, sun2021lrpfinetuning}.

Despite its theoretical advantages, LRP remains constrained by practical limitations tied to module-specific rules, severely limiting adaptability to rapidly evolving architectures like Transformers \cite{achtibat2024attnlrp, otsuki2024lrpresidualconnections}. Modern libraries like \texttt{iNNvestigate} \cite{alber2018innvestigate}, \texttt{Captum} \cite{kokhlikyan2020captum}, and notably \texttt{Zennit} \cite{anders2021zennit} attempt to mitigate this with automated differentiation, but remain module-centric: new architectures typically need new composite rules and custom configurations. Even specialized advances like \texttt{AttnLRP} \cite{achtibat2024attnlrp} (built on Zennit) require manual adaptation for new model updates (e.g., Llama 2 vs. 3), creating friction and overhead for continued use.


\begin{figure}[h]
  \centering
  \includegraphics[width=1.0\textwidth]{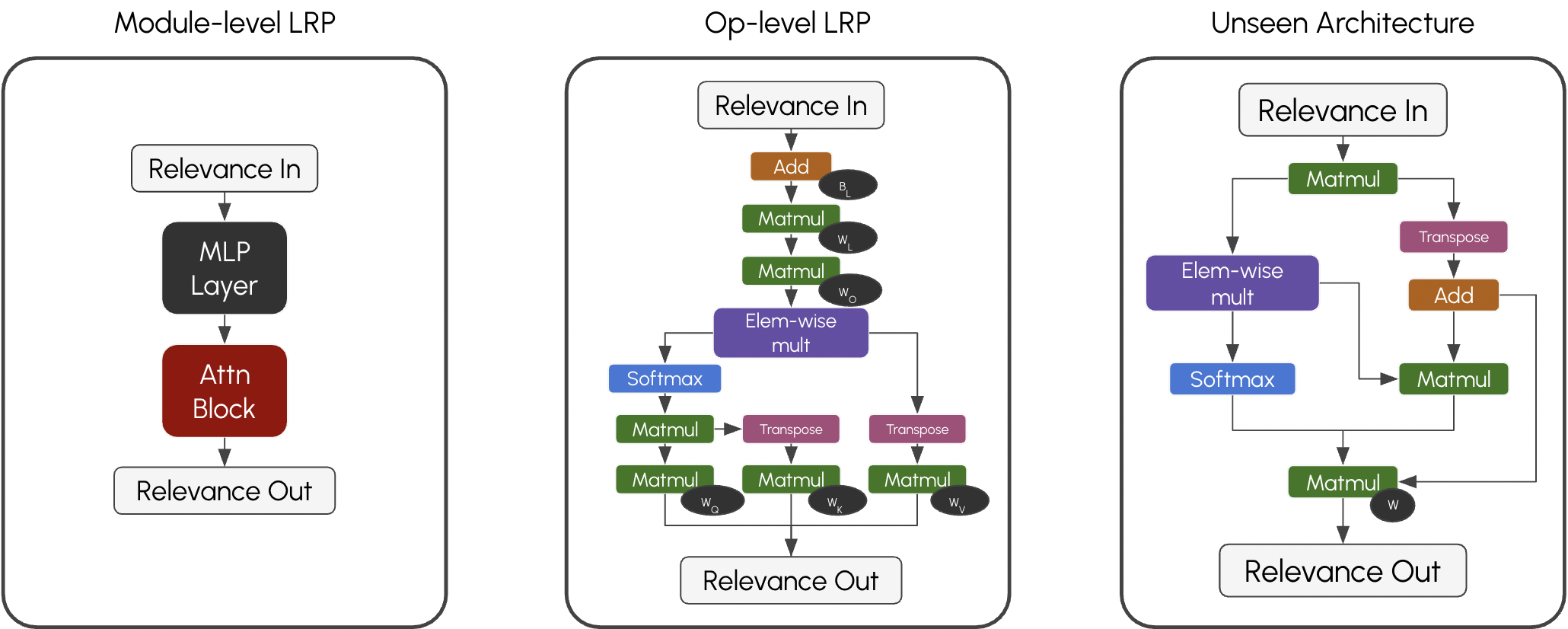}
  \caption{Relevance flow diagrams for sample transformer blocks demonstrate the contrast in design. Low-level implementation allows Dynamic LRP adaptation to any arrangement of covered operations, such as the toy picture on the right. Note that as relevance flows from output to input, the arrow direction opposes inference.}
  

  \label{fig:opLevel-LRP}
\end{figure}

To address these limitations, we introduce \textbf{DynamicLRP}, which shifts the focus from a module-centric to a primitive-based formulation by redefining the propagation rules at the level of atomic tensor operations (e.g., addition, matrix multiplication). Analogous to generalized gradient computation, DynamicLRP implements LRP rules for 47 primitives and propagates relevance directly through the computation graph, making the method inherently model-agnostic and free of model-specific configuration (Figure~\ref{fig:opLevel-LRP}). This design is enabled by the \textbf{Promise System}, a novel deferred propagation mechanism that resolves the challenge of missing operand information during non-gradient computations within auto-differentiation frameworks (Sec~\ref{sec:promise-system}). Lastly, we show that our approach matches or exceeds specialized implementations in attribution quality (Sec~\ref{sec:experiments}) across popular portable vision (ViT-b-16 \cite{dosovitskiy2021vit}, VGG16 \cite{simonyan2015vgg}) and language (LLaMA \cite{dubey2024llama3}, BERT \cite{devlin2019bertpretrainingdeepbidirectional}, RoBERTa \cite{liu2019roberta}, T5-Flan \cite{lee2023flant5squadv2}) models \textit{while} maintaining efficiency on models into the billion-parameter scale.   


\section{Preliminaries}
\subsection{Classic LRP via Taylor Approximation and LRP variants}
\label{methods:taylor}

The classic LRP formulation as presented by \citet{bach2015pixel} derives LRP using what is called the deep Taylor decomposition framework. This formulation is functional, treating the layers as mathematical operators, and thus it is devoid of the structural assumptions of any given architecture. By taking the first-order Taylor expansion of the function at some root point $x_0$ such that $f(x_0) = 0$, and invoking the proposed LRP relevance conservation property, they define a recursive expression for the relevance of neurons $x_j$ at layer $l$ w.r.t. neurons $x_i$ at layer $l - 1$. For a detailed derivation, see Appendix~\ref{appendix:taylor}

\textbf{The $\epsilon$-LRP Rule}
Assuming the bias contribution is negligible or handled separately (see \cite{achtibat2024attnlrp} for a detailed discussion on bias absorption), the input relevance can be computed by summing over all output neurons and introducing a stabilizing factor $\epsilon$ to properly handle vanishing denominators. For a linear layer $z_j = \sum_i x_i w_{ji} + b_j$ with $x_0 = 0$:
\begin{equation}
    R_{i \leftarrow j}^{l - 1} = \frac{x_i w_{ji}}{z_j + \epsilon} R_j^l
    \label{eq:assignment}
\end{equation}
This recovers the standard $\epsilon$-LRP rule. In matrix form, efficiently implementable on GPUs:
\begin{equation}
    R^{l-1} = x \odot (W^\top \cdot (R^l \oslash (z + \epsilon)))
\end{equation}
where $\odot$ denotes the Hadamard product and $\oslash$ element-wise division.

\textbf{The $\gamma$-LRP Rule}
Attributions can become noisy when using the $\epsilon$-LRP rule, particularly in scenarios where the denominator $z_j + \epsilon$ becomes dominated by the stabilizing factor $\epsilon$, resulting in locally insensitive relevance scores. The $\gamma$-LRP rule \cite{montavon2019lrp}, can be used to address this by introducing the hyperparameter $\gamma$ to act as a multiplicative booster for positive neuron contributions:

\begin{equation}
  R_{i \leftarrow j}^{l - 1} = \frac{x_i \cdot (w_{ji} + \gamma w_{ji}^+)}{\sum_i x_i \cdot (w_{ji} + \gamma w_{ji}^+)} R_j^l
\end{equation}

The Deep Taylor Decomposition framework has been recently used to derive propagation rules for all components in transformer-based architectures~\cite{vaswani2017attention}. In their formulation, Achtibat et al.~\cite{achtibat2024attnlrp} successfully derive rules for both bilinear matrix multiplications and softmax operations used in attention mechanisms. Our work builds on these insights to enable the reuse of these primitives for novel architectures.

\subsection{Computation Graphs in Deep Learning Frameworks}

\label{sec:computation-graph}
Our attribution algorithm relies on the \textbf{Backward Computation Graph}, used by most deep learning frameworks (e.g., PyTorch, TensorFlow) to facilitate automatic differentiation. Informally, the backward computation graph of a neural network is a Directed Acyclic Graph (DAG) $G = (V, E)$ whose nodes $V$ correspond to operations executed during the forward pass, and whose directed edges $(u, v) \in E$ indicate that the output of operation $v$ was used as an input to operation $u$ in the forward pass, so that edges are oriented from outputs toward inputs.


The in-degree and out-degree of a node denote the number of operations that consumed its output and the number of inputs it consumed, respectively. Hence, the model output serves as the \textit{source node} \textit{i.e.,} with zero incoming edges or in-degree of 0, while model parameters and data inputs are \textit{sink nodes} \textit{i.e.,} with zero outgoing edges or out-degree of 0. For any node $u$, its outgoing adjacency set $outadj(u) = \{v | (u, v) \in E\}$ represents the set of operations that produced the inputs for $u$ (i.e., its downstream dependencies in the backward pass), the converse being its incoming adjacency set $inadj(u) = \{v | (v, u) \in E\}$.
Our method leverages this structure to propagate relevance scores in a manner analogous to gradient computation, ensuring that at each step, the necessary information is available for accurate attribution.

\section{Methods}
\label{sec:op-level-lrp}
We motivate shifting the level of abstraction of LRP from layer/module-wise to operation-wise through careful consideration of the modern-day deep learning
paradigm and the practices used by existing LRP implementations. We define a module as an object lying at the maximal granularity
within the model structure of a deep learning framework (e.g. Linear, Convolution, Attention, LayerNorm).
While many modules are defined by a single operation, such as Linear, Convolution, Softmax, etc., certain modules compound many elementary tensor operations, which are not captured at the granularity of modules. There is thus redundancy in redefining propagation steps for tensor primitives within different module composite rules. The other consequence is that there is no upper limit for the number of unique module types, and each requires a new backpropagation rule. This is unsustainable in the long-term. By shifting to the level of operations, the rules for a set of operations needs to only be defined once before all architectures (and thus modules) that decompose to that set of operations are supported.

\subsection{Topological Propagation and the Problem of Missing Activations}
\label{sec:topo-propagation}
Given the backward computation graph $G = (V, E)$ defined in Section~\ref{sec:computation-graph}, our goal is to propagate relevance from the source (model output) to the sinks (inputs/parameters). Note that this operation-wise view contrasts with the classic layer-wise view (Section \ref{methods:taylor}) but is mathematically equivalent, as layers are simply subgraphs of operations. Conservation of relevance applies locally at each node (i.e. $\sum_{v \in outadj(u)} R_{u \to v} = \sum_{p \in parents(u)} R_{p \to u}$)

Ideally, LRP should traverse $G$ in a way such that a node $v$ processes relevance only after it has received all relevance contributions from its parents in $G$, so we must visit nodes in an order that respects dependencies. Specifically, a node $v$ should only compute and distribute its relevance once it has aggregated relevance from all nodes $u$ that depend on it (i.e., all $u \in inadj(v)$). This dependency constraint is satisfied by traversing $G$ in some topological order. We present the graph traversal algorithm used for this purpose in Algorithm~\ref{alg:topological-traversal}.

\iftoggle{arxiv}{
\begin{algorithm}[!h]
\small
  \caption{Topological Traversal of Backward Graph}
  \label{alg:topological-traversal}
  \begin{algorithmic}[1]
  \State \textbf{Input:} Backward Graph $G = (V, E)$
  \State Initialize $S$ as a stack containing all nodes with $indegree(v) = 0$
  \State Initialize $pending(v) = indegree(v) \quad \forall v \in V$
  \While{$S$ is not empty}
    \State Pop node $v$ from $S$
    \State \Call{Process}{$v$} \Comment{Apply operation-specific propagation}
    \For{$w \in outadj(v)$}
      \State $pending(w) \gets pending(w) - 1$
      \If{$pending(w) == 0$}
        \State Push $w$ onto $S$
      \EndIf
    \EndFor
  \EndWhile
  \end{algorithmic}
\end{algorithm}

}{}

While the topological traversal ensures the correct \textit{order} of operations, propagating the relevance will fail in some cases due to the \textit{information} available at each node. Consider the propagation rule for addition $c = a + b$, which distributes the relevance proportional to the contribution magnitude:
\begin{equation}
  R_{a} = R_{c} \times |a|/(|a| + |b|) \quad\text{and}\quad R_{b} = R_{c} \times |b|/(|a| + |b|)
\end{equation}
To compute $R_a$ and $R_b$, we strictly require the values of operands $a$ and $b$. However, modern automatic differentiation frameworks optimize memory by discarding $a$ and $b$ after the forward pass since $\frac{\partial c}{\partial a} = \frac{\partial c}{\partial b} = 1$. Thus, even if the topological traversal visits the Addition node at the correct time, the operand tensors required for the attribution rule are missing from the graph. 

\subsection{The Promise System}
\label{sec:promise-system}
When we reach a Node in traversal where propagation would halt from such a case, we instantiate a \textbf{Promise},
which defers the propagation computations and retrieves the missing tensors from further down in the graph. Conceptually, a Promise acts as a \textit{placeholder} for missing activations. When an operation requires an unavailable tensor to compute relevance, the Promise suspends relevance propagation but continues traversing through the graph to recover the needed values. Once these are found, the relevance propagation at the problematic Node is computed, and the propagation ``catches up'' across all the Nodes that were traversed during the search, which are recorded to avoid backtracking and retraversal of the computation graph. We will now formally define this system.\\

\subsubsection{Basic Promise Mechanism}

\begin{definition}
\label{def:promise}
Let $\mathcal{T}$ be the space of tensors. We define a Promise $\mathcal{P}$ at a Node $u$ as a tuple $\mathcal{P}_u = (R_{out}, \rho, \chi, \mathcal{A}, \mathcal{R}_{in})$ where:
\begin{itemize}
    \item $R_{out} \in \mathcal{T}$ is the accumulated output relevance.
    \item $\rho \in \{0, 1\}$ is the \textit{readiness flag}, where $\rho=1$ indicates that all missing arguments have been retrieved.
    \item $\chi \in \{0, 1\}$ is the \textit{completeness flag}, where $\chi=1$ indicates that relevance has been distributed to all the inputs.
    \item $\mathcal{A} = (a_1,\dots, a_k)$ is a sequence of retrieved argument tensors, where $ a_i \in \mathcal{T}$.
    \item $\mathcal{R}_{in} = (r_1, \dots, r_k)$ is a sequence of computed input relevances, where $r_i \in \mathcal{T}$.
\end{itemize}
\end{definition}

The node $v$ where a Promise is instantiated is termed the \textbf{Origin Node}. To resolve the missing arguments $\mathcal{A}$, the Promise spawns $k$ \textbf{Promise Branches}, where the $i$-th branch is responsible for retrieving $a_i$. As an illustration, consider the following subgraph requiring a single Promise Branch, where red nodes do not store forward activations but require them for relevance propagation, green nodes do not store any forward activations and do not require them for relevance propagation, and yellow nodes store forward activations: 
\vspace{10.0pt}
\begin{center}
\resizebox{0.5\linewidth}{!}{%
\begin{tikzpicture}
  \node[node, purple] (node1) {A};
  \node[node] (node2) [right=of node1] {B};
  \node[node] (node3) [right=of node2] {C};
  \node[node] (node4) [right=of node3] {D};
  \node[node, orange] (node5) [right=of node4] {E};

  \draw[->] (node1.east) -- (node2.west);
  \draw[->] (node2.east) -- (node3.west);
  \draw[->] (node3.east) -- (node4.west);
  \draw[->] (node4.east) -- (node5.west);
  \begin{scope}[transform canvas={yshift=.3cm}]
    \draw[->, shorten <= -0.5cm, shorten >= -0.5cm] (node5.north) -- (node1.north) node[midway, above] {Model forward};
  \end{scope}
  \draw [decorate,decoration={brace,amplitude=5pt,mirror,raise=2ex}]
    (node1.south west) -- (node4.south east) node[midway,yshift=-3em]{Do not store forward activations};
\end{tikzpicture}%
}
\end{center}

Note that without node A's forward activation, it is not possible to continue relevance propagation. However, since there exists a node, e.g. node $E$, further down the traversal path whose forward output is available, we can reapply the transformations of the nodes in between A and E to reconstruct A's missing forward activation and propagate the relevance. For deferred propagation, each branch traverses the graph until it encounters a node $u$ whose forward output is locally retrievable; We term this the \textbf{Arg Node}.

\textbf{Proposition 1:} For any node $v$ in a computation graph, there exists at least one Arg Node which is a descendent of $v$ in the computation graph.

The proof is provided in Appendix~\ref{sec:lrp-complexity}.

Let $\pi = (v, w_1, \dots, w_m, u)$ be the traversal path from Origin $v$ to Arg Node $u$. For each intermediate node $w_j$, we record its forward operation $f_{w_j}$ and relevance rule $b_{w_j}$, forming two functional chains:
a Forward Chain $\Phi_{v \leftarrow u} = f_{w_1} \circ \dots \circ f_{w_m}$ and a Backward Chain $\Psi_{u \to v} = b_{w_m} \circ \dots \circ b_{w_1}$. Upon reaching $u$, we retrieve its output $y_u$ and reconstruct the missing argument $a_i \in \mathcal{A}$ via $a_i = \Phi_{v \leftarrow u}(y_u)$.
Once all arguments in $\mathcal{A}$ are resolved, we set $\rho_{u}=1$, and the Origin Node $v$ computes the local relevance distribution $\mathcal{R}_{in}$ using $\mathcal{A}$ and $R_{out}$. Finally, the input relevance $r_i \in \mathcal{R}_{in}$ is fast-forwarded back to $u$ via the backward chain: $R_{u} = \Psi_{u \to v}(r_i)$. 
This mechanism effectively defers propagation until dependencies are satisfied without backtracking.
We provide a full example in Appendix~\ref{appendix:basic-promise-example}.

\subsubsection{Promise Nesting and Trees}
Now that we have introduced Promises as placeholder relevance inputs, we must also consider how any given Node will handle a Promise Branch input.
\begin{definition}
A Promise-Generating Operation is an autograd node $v_p$ that may potentially create a new Promise object during propagation.
We categorize all Promise-Generating Operations as one of the following:
\begin{itemize}
  \item \textbf{Strict: }$v_p$ will always create a new Promise object, independent of its relevance input types.
  \item \textbf{Dependent: }$v_p$ will create a new Promise object only if at least one of its relevance inputs is a Promise, otherwise it returns a tensor relevance.
\end{itemize}

Note that whether a Dependent Promise-Generating Node instantiates a new Promise or not is architecture-specific.

When a Promise-Generating Node $v$ receives a Promise Branch $p$ as input, $v$ creates a new Promise $\mathcal{P}'$ and nests all of $\mathcal{P}'$'s branches as children of $p$ via parent-child connections, forming a Promise Tree.
Promise Trees resolve bottom-up, starting with leaf Promise Branches, which encounter Arg Nodes, and forward-chain to reconstruct their ancestor Promises' activations via parent connections. Once all activations of the Promise at the root of the Tree have been recovered, the relevance is propagated back down the Tree via child connections and backward chains. See Appendix~\ref{appendix:promise-trees} for the formal resolution algorithm.
\end{definition}

\subsubsection{Promise Deadlock}
Promises can cause circular dependencies between Promise Branches. Consider the following graph, reflecting a pattern found in residual connections, crucial in transformer networks and ResNets. Let A be a Promise-Generating Node, and let B and E be Arg Nodes, coloured orange.
Nodes coloured red are traversed, green are untraversed.

Recall that we may only consider a node for traversal if all of its in-neighbours have propagated to it. In Step \textit{(a)}, A creates a Promise $\mathcal{P}$, propagates its Promise Branches
$p_1$, $p_2$ to B and D. We traverse to B, as all its inputs have now landed. 

B is an Arg Node, so the Promise Branch will trigger a forward-chain to obtain the activation at A and save it.
In Step \textit{(b)}, we now stall traversal and enqueue B until the Promise is complete. Next, in Step \textit{(c)}, we traverse to D, which received Promise Branch $p_2$ from A as input.

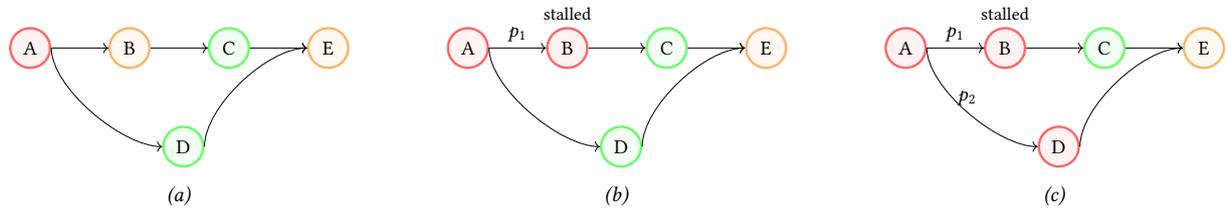
\begin{figure}[ht]
    \centering
    \begin{subfigure}[b]{0.28\textwidth}
        \centering
        \resizebox{\linewidth}{!}{%
        \begin{tikzpicture}
        \node[node, red] (node1) {A};
        \node[node, orange] (node2) [right=of node1] {B};
        \node[node] (node3) [right=of node2] {C};
        \node[node] (node4) [below=of node3, xshift=-8mm] {D};
        \node[node, orange] (node5) [right=of node3] {E};
        
        \draw[->] (node1.east) -- (node2.west);
        \draw[->] (node2.east) -- (node3.west);
        \draw[->] (node3.east) -- (node5.west);
        \draw[->] (node1.east) .. controls +(down:7mm) and +(left:7mm) .. (node4.west);
        \draw[->] (node4.east) .. controls +(up:7mm) and +(left:7mm) .. (node5.west);
        \end{tikzpicture}%
        }
        \caption{}
    \end{subfigure}
    \hfill 
    \begin{subfigure}[b]{0.28\textwidth}
        \centering
        \resizebox{\linewidth}{!}{%
        \begin{tikzpicture} 
        \node[node, red] (node1) {A};
        \node[node, red, label={\small stalled}] (node2) [right=of node1] {B};
        \node[node] (node3) [right=of node2] {C};
        \node[node] (node4) [below=of node3, xshift=-8mm] {D};
        \node[node, orange] (node5) [right=of node3] {E};
        
        \draw[->] (node1.east) -- (node2.west) node[midway, above] {$p_1$};
        \draw[->] (node2.east) -- (node3.west);
        \draw[->] (node3.east) -- (node5.west);
        \draw[->] (node1.east) .. controls +(down:7mm) and +(left:7mm) .. (node4.west);
        \draw[->] (node4.east) .. controls +(up:7mm) and +(left:7mm) .. (node5.west);
        \end{tikzpicture}%
        }
        \caption{}
    \end{subfigure}
    \hfill
    \begin{subfigure}[b]{0.28\textwidth}
        \centering
        \resizebox{\linewidth}{!}{%
        \begin{tikzpicture} 
        \node[node, red] (node1) {A};
        \node[node, red, label={\small stalled}] (node2) [right=of node1] {B};
        \node[node] (node3) [right=of node2] {C};
        \node[node, red] (node4) [below=of node3, xshift=-8mm] {D};
        \node[node, orange] (node5) [right=of node3] {E};
        
        \draw[->] (node1.east) -- (node2.west) node[midway, above] {$p_1$};
        \draw[->] (node2.east) -- (node3.west);
        \draw[->] (node3.east) -- (node5.west);
        \draw[->] (node1.east) .. controls +(down:7mm) and +(left:7mm) .. (node4.west) node[midway, above] {$p_2$};
        \draw[->] (node4.east) .. controls +(up:7mm) and +(left:7mm) .. (node5.west);
        \end{tikzpicture}%
        }
        \caption{}
    \end{subfigure}
        \caption{Illustration of Promise Deadlock. (a) Origin node A distributes promise branches to retrieve activations. (b) Branch $p_1$ reaches Arg Node B, which stalls waiting for completion. (c) Branch $p_2$ reaches Arg Node D, which propagates to E. However, E depends on C, which depends on the stalled B, creating a cycle preventing Promise completion.}
\end{figure}
\vspace{-5pt}

The solution, in brief, is to break our traversal heuristic upon such situations, and allow Promise Branches to propagate until they reach Arg Nodes. The specific manner in which this is done is explained fully in Appendix~\ref{appendix:pre-promises}, but a simplified version of the propagation algorithm that considers all of them is presented in Algorithm~\ref{alg:promise-propagation}. Note that this allows delayed computation to occur without the graph traversal pointer having to backtrack or revisit any node not on the current traversal frontier, a property which we prove in Appendix~\ref{appendix:unique-chains} and allows us to define clear upper bounds on the computational and memory overhead introduced by the Promise system in Appendix~\ref{sec:lrp-complexity}.

\begin{algorithm}[!h]
\small
  \caption{Operation-Level Relevance Propagation With Promises}
  \label{alg:promise-propagation}
  \begin{algorithmic}[1]
    \State \textbf{Input:}Model output \texttt{hidden\_states}, Relevance attribution target \texttt{target\_node}, Incoming relevance $R_{out}$
    \State Initialize $\texttt{stack} \gets [\texttt{hidden\_states.grad\_fn}]$
    \State Initialize $\texttt{heuristic\_break\_queue} \gets [\hspace{0.2em}]$
    \State Initialize $\texttt{input\_tracker} \gets \{ \texttt{node} : [\hspace{0.2em}] \text{ for all nodes in graph} \}$
    \State Initialize $\texttt{nodes\_pending} \gets \{ \texttt{node} : \text{in-degree of } \texttt{node} \text{ for all nodes in graph}\}$
    \Repeat
      \If{\texttt{heuristic\_break\_queue} is not empty}
        \State $\texttt{curnode} \gets \texttt{heuristic\_break\_queue.dequeue()}$
        \State Create new Pre-Promise $P_{pre}$, assign Promise Branch in \texttt{input\_tracker[curnode]} as parent \Comment{See line~\ref{lst:line:prepromise-condition}}
        \State Store $P_{pre}$ in \texttt{curnode} metadata
        \State $\texttt{curnode\_in\_rel} \gets P_{pre}$
      \Else
        \State $\texttt{curnode} \gets \texttt{stack.pop()}$ \Comment{Process next node}
        \State $\texttt{curnode\_in\_rel} \gets \text{Relevance accumulated at } \texttt{curnode}$
        \If{\texttt{curnode} is \texttt{target\_node}} \Comment{Target reached}
          \State \Return $R_{in} := \texttt{curnode\_in\_rel}$
        \EndIf
      \EndIf
      \If{\texttt{curnode} requires Promise} \Comment{Defer computation}
        \State Create new Promise $P$ using \texttt{type(curnode)} and assign parent-child connections if applicable
        \State Assign each child of \texttt{curnode} its respective branch of $P$ via \texttt{input\_tracker[child]}
      \Else
        \State Execute propagation function for \texttt{curnode} \Comment{Applies LRP rule}
        \State Assign each child of \texttt{curnode} its respective split of the relevance via \texttt{input\_tracker[child]}
      \EndIf
      \For{each \texttt{child} of \texttt{curnode}}
        \State Decrement \texttt{nodes\_pending[child]} by 1 \Comment{Update dependencies}
        \If{\texttt{nodes\_pending[child]} is 0}
          \State Push \texttt{child} to \texttt{stack} \Comment{Enqueue ready children}
        \ElsIf{\texttt{input\_tracker[child]} contains a Promise Branch} \label{lst:line:prepromise-condition}
          \State Add \texttt{child} to \texttt{heuristic\_break\_queue}
        \EndIf
      \EndFor
    \Until{\texttt{stack} is empty} \Comment{All nodes processed}
      \State \texttt{curnode} $\gets$ \texttt{stack.pop()} \Comment{Process next node}
    \State \Return Input relevance $R_{in}$
  \end{algorithmic}
\end{algorithm}

\subsubsection{Caching Promise Structures}
\label{sec:promise-caching}
The previously described methodology covers only the first pass of the LRP algorithm. Repeated runs of LRP on the same model architecture yield an identical set of saved Promise computation paths. So, we cache these paths to improve the algorithm's scalability. Once a Promise's computation path has been defined, it can effectively be reduced to a single node, eliminating the graph traversal overhead of its internal nodes. Finally, consider that if we know the Promise chains beforehand and can reduce them to single nodes, then the order in which we perform computations becomes fully deterministic and independent of edge connections, reducing the graph traversal overhead to $O(n)$, even when promise deadlock is encountered during initial traversal.

\section{Experimental Setup}
\label{sec:experiments}

In this work, we aim to demonstrate that DynamicLRP produces faithful explanations for modern architectures of varying scale.  In attempt to measure faithfulness, we adapt the evaluation strategies from \cite{achtibat2024attnlrp} using classification and exctractive question-answering tasks.

\subsection{Classification Tasks}
\label{sec:experiments:classification}

Input perturbation, as explored and advanced by \citet{blucher2024pixelflipping}, \citet{deyoung2020eraserbenchmarkevaluaterationalized}, and \citet{schulz2020abpc}, has become a widely-used standard for measuring attribution faithfulness. This method leverages the model itself as an oracle to validate the feature ranking generated by an attribution method. By systematically removing or replacing input features (e.g., tokens or pixels) and observing the degradation in performance with respect to a baseline, we can verify if a feature's attributed ``importance'' is in-line with a model's prediction. We quantify this using the Area Between Perturbation Curves (ABPC), ranking features by importance and iterating in both orders:
\begin{enumerate}
    \item \textbf{Most Relevant First (MoRF)}: We occlude the most important features first. A faithful attribution should result in a steep decline in the model's confidence for the true label, as crucial information is removed.
    \item \textbf{Least Relevant First (LeRF)}: We occlude the least important features first. A faithful attribution should result in a flat curve or minimal decrease (potentially even an increase) in confidence, as irrelevant features are removed.
\end{enumerate}
The ABPC is defined as $\text{AUC}(\text{LeRF}) - \text{AUC}(\text{MoRF})$, forming a unified metric where higher values indicate better faithfulness (higher stability under LeRF and sharper drop under MoRF). Other metrics related to faithfulness are \textit{comprehensiveness}, defined as $\text{AUC}(\text{baseline}) - \text{AUC}(\text{MoRF})$ and \textit{sufficiency}, defined as $\text{AUC}(\text{baseline}) - \text{AUC}(\text{LeRF})$.

\paragraph{Visual Classification Setup} We use the the base VGG16 model  \citep{simonyan2015vgg} evaluated on the ImageNette-320 dataset \citep{howard2019imagenette}, and the   base ViT-b-16 \citep{dosovitskiy2021vit} architecture on the CIFAR-10 dataset \citep{krizhevsky2009learning}. For perturbation, we occlude features by replacing them with a Gaussian-blurred version of the image ($\text{kernel size}=51$, $\sigma=20$). For ViT-b-16, we occlude the top-4 $16 \times 16$ patches (pooled by max relevance) at each step. For VGG16, we occlude 1024 individual pixels per step. 

\paragraph{Text Classification Setup}
Following the experimental setup in \cite{achtibat2024attnlrp}, we first use a finetuned Llama-3.2-1B model \citep{dubey2024llama3} on the IMDb dataset \citep{maas-IMDB} for sentiment analysis. We evaluate explanations by determining the relevance of input tokens to the predicted class. For language modelling, we use the base model on the English subset of the Wikipedia \footnote{\cite{wikidump}}  dataset and frame the task as next-token prediction. We compute attribution for the ground-truth next token given a context window of 512 tokens. In both settings, we apply LRP-$\gamma$ with $\gamma_{linear}=1.0$ and compare against baseline attribution methods using MoRF and LeRF occlusion metrics after masking one token per step.

\subsection{Extractive Question-Answering Tasks}
\label{sec:experiments:qa}

Extractive Question-Answering (EQA) tasks differ structurally from classification: the model identifies a span within the input context as the answer. In this setting, the MoRF/LeRF strategy can be redundant, as the tokens constituting the predicted span are by definition necessary for the model to output that specific span. Therefore, we simplify the measurement to checking if the highest-ranked tokens are within the predicted answer span. We evaluate faithfulness via the accuracy of the top-1 attributed token falling within the label golden span and the model's predicted span. We use finetuned versions of RoBERTa-large \citep{liu2019roberta} and Flan-T5-large \citep{chung2022flant5} on the SQuAD-v2 dataset \citep{rajpurkar2018squadv2,chan2021robertasquadv2, lee2023flant5squadv2}. We include all answerable examples in the validation set, skipping only those where the model predicts an invalid span (start index > end index). We run LRP-$\gamma$ with $\gamma_{linear} = 0.001$.  We mask the relevance of tokens found in the question before taking the top-1 attributed token. 
\section{Results}
 As can be seen from tables~\ref{results-table}, \ref{vit-table}, and \ref{vgg-table}, DynamicLRP consistently outperforms generic attribution methods in vision tasks and matches or surpasses model-specific AttnLRP implementation (Figure~\ref{fig:garbage-truck}). 

\begin{figure}[!ht]
    \centering
    \begin{subfigure}[b]{0.28\textwidth}
        \centering
        \includegraphics[width=0.8\linewidth]{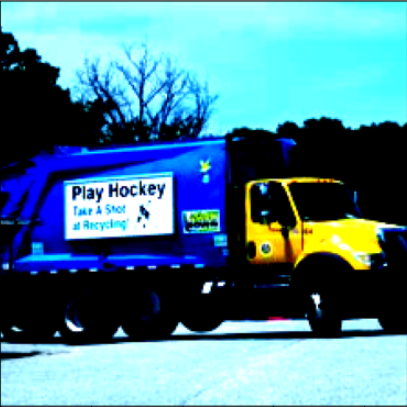}
        \caption{}
    \end{subfigure}
    \hfill 
    \begin{subfigure}[b]{0.28\textwidth}
        \centering
        \includegraphics[width=.8\linewidth]{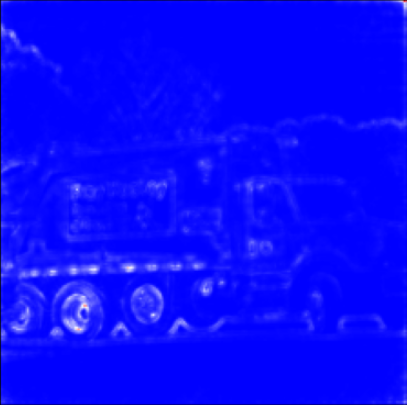}
        \caption{}
    \end{subfigure}
    \hfill
    \begin{subfigure}[b]{0.28\textwidth}
        \centering
        \includegraphics[width=.8\linewidth]{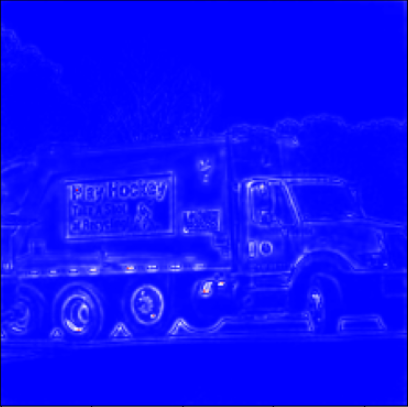}
        
        \caption{}
    \end{subfigure}
        \caption{Visual representation of attributions produced by different LRP algorithms for the predicted class truck by a {ViT-b} architecture. (a) Original image (b) Zennit LRP attributions (c) DynamicLRP attributions}
\label{fig:garbage-truck}
\end{figure}



In Language tasks, DynamicLRP demonstrates faithful attribution in all cases, achieving the best score in three out of five language tasks and ranking second in the others. When compared on the larger model, LLaMA3.2-1B, we note that the specialized Attn-LRP implementation outperforms DynamicLRP. We attribute this difference to relevance distribution rather than signal quality. DynamicLRP tends to spread relevance more broadly across influential euqally influential tokens, whereas AttnLRP produces sharper peaks that are favoured by rank-based metrics. Qualitatively, however, both methods yield highly consistent attribution signals and identify the same semantic features as important (Figure~\ref{fig:llama-plots}). For extractive question answering tasks, DynamicLRP outperforms all other methodologies. It is worth mentioning that Top-1 Token-to-Gold-Span(TGS) accuracy is tightly coupled to model predictive performance and therefore confounds attribution quality with task accuracy, requiring careful interpretation alongside metrics such as F1 and Exact Match (Table~\ref{squadv2-table}). To disentangle these effects, we consider Top-1 Token-to-Predicted-Span (TPS), which measures whether attributions identify the tokens actually used by the model for its prediction. LRP-based methods maintain consistently high TPS, while gradient-based methods degrade sharply on erroneous predictions. This indicates that LRP preserves the importance of the predicted span throughout the computation path, even when the model fails. The robustness of DynamicLRP on variuos models highlights its practical value as an architecture-agnostic diagnostic tool, particularly during model development, where explaining failures is often more informative than explaining successes.
\begin{table}[!thp]
\scriptsize
\caption{Quality of explanations across attribution methods. For vision and language classification tasks(ImageNette, CIFAR, IMDb and Wikipedia), faithfulness is measured by the Area Between MoRF and LeRF perturbation curves. For question answering (SQuADv2) tasks, we report Top-1 Token-to-Golden-Span (TGS) and Token-to-Predicted-Span (TPS). Higher is better; best results in \textbf{bold}.}
\vspace{-0.35cm}
\begin{center}
  \label{results-table}
  \begin{tabular}{lccccccc}
    \toprule
    \textbf{Task type} & \multicolumn{2}{c}{Image Class.} & Text Class. & Causal LM & \multicolumn{3}{c}{Extractive QA} \\
    \midrule
    \textbf{Method} & \textbf{VGG} & \textbf{ViT-b-16} & \textbf{LlaMa3.2-1B} &\textbf{LlaMa3.2-1B} & \textbf{BERT} & \textbf{RoBERTa-L} & \textbf{T5-Flan-L} \\
    & ImageNette-320 & CIFAR-10 & IMDB  & Wikipedia & SQuADv2 & SQuADv2 & SQuADv2 \\
    \midrule
    
    IG & 1.39 & 0.10 & 1.77 & 1.68 & 0.438 (0.558) & 0.353 (0.134) & 0.798 (0.127) \\
    GradSHAP &1.62 & 0.074 & 2.61 & 1.59 & 0.419 (0.527) & 0.374 (0.143) & 0.795 (0.157) \\
    Zennit (Attn-LRP) &  1.69 & \textbf{1.46} & \textbf{3.73} & \textbf{3.16} & 0.517 (0.805) & N/A$^{*}$ & N/A$^{*}$\\
    \midrule
    DynamicLRP (ours) & \textbf{1.77} & \textbf{1.46} & 3.18 & 2.59 & \textbf{0.579} (\textbf{0.986}) & \textbf{0.937} (\textbf{0.889}) & \textbf{0.951} (\textbf{0.986}) \\
    \bottomrule
  \end{tabular}
\end{center}
$^{*}$ The architecture is not supported by the respective method.
\end{table}

\subsection{Architecture-Specific Factors on LRP Computational Overhead}
DynamicLRP's computational overhead depends on two key architectural traits: Promise Depth ($\delta$) and Promise Density ($\rho$). Promise depth $\delta$ of an architecture is the maximum graph distance to an Arg Node for all Promise Branches. The Promise density is the ratio of promise-generating operations to total operations, $\rho = \frac{|V_P|}{|V|}$. Table~\ref{architecture-promises-table} reports these values across architectures, providing empirical insight to the runtime and memory complexity. Notably, model variants of Flan-T5-large show that increased model complexity (e.g., more autograd nodes) does not necessarily scale up Promise density. This suggests that complex training procedures or architectural expansions do not disproportionately increase the ``promise overhead'' implying favourable scalability for our method.

\begin{table}[!h]
\scriptsize
\begin{center}
  \caption{Architecture-specific values relating to the Promise System's impact on LRP efficiency. We include a separate row for the finetuned (FT) version of Flan-T5-large, which uses LoRA adapters, changing some of these values drastically.
  Depth is in terms of either the number of Transformer layers (T) or convolutional layers (C).}
  \label{architecture-promises-table}
  \begin{tabular}{ lccccccc  }
    \toprule
    \textbf{Model} & \textbf{$d_{model}$} & \textbf{Depth} & \textbf{Promises} & \textbf{Internal nodes} & \textbf{$\delta$} & \textbf{$\rho$} & \textbf{Total nodes} \\
    \midrule
    RoBERTa-large & 1024 & 24 (T) & 243 & 48 & 438 & 0.166 & 1,465 \\ 
    Flan-T5-large & 1024 & 48 (T) & 798 & 348 & 980 & 0.140 & 5,713 \\ 
    Flan-T5-large (FT) & 1024 & 50 (T) & 847 & 1,524 & 2,972 & 0.0777 & 10,897 \\ 
    ViT-b-16 & 768 & 1 (C) 12 (T) & 224 & 27 & 227 & 0.288 & 779 \\ 
    VGG16 & 512 & 13 (C) & 32 & 0 & 32 & 0.356 & 90 \\ 
    \bottomrule
  \end{tabular}
\end{center}
\end{table}
\vspace{-20pt}

\subsection{Operation Coverage}
 By targeting a bounded set of standard tensor operations, our operation-level approach aims for broad coverage. Table~\ref{coverage-table} highlights our coverage on a representative subset of architectures, achieving near-perfect node coverage across diverse modalities without model-specific code. We achieve an aggregate 99.99\% node coverage (31,464/31,465 nodes) across a broad suite of 15 models spanning vision, language, audio, and state-space architectures. See Appendix~\ref{sec:appendix_extra_results} for the full suite of 15 models.
\begin{table}[!ht]
\scriptsize
\begin{center}
  \caption{DynamicLRP operation coverage for selected architectures. }
  \label{coverage-table}
  \begin{tabular}{ lcccc }
    \toprule
    \textbf{Model} & \textbf{Modality} & \textbf{Covered Nodes} & \textbf{Unique Ops covered} \\
    \midrule
    ViT-b-16 \cite{dosovitskiy2021vit} & Vision & 779/779 & 16/16 \\
    Wav2Vec2-xls-r-300m \cite{babu2021xlsr} & Audio & 2,021/2,022 & 18/19 \\
    Llama3.2-1B \cite{dubey2024llama3} & Language & 1,787/1,787 & 24/24 \\
    Mamba-130m \cite{gu2024mamba} & Language & 3,421/3,421 & 26/26 \\
    \bottomrule
  \end{tabular}
\end{center}
\end{table}
\vspace{-20pt}

\section{Discussion}
Although current implementation is specific to PyTorch, our principal intellectual contribution is the framework-agnostic Promise-based design. Auto-differentiation engines universally rely on computation graphs to store operation-level data. Since the mathematical requirements for differentiating functions are invariant (the set of ``givens'' required for a derivative is a property of the function, not the software), the prerequisite data structures are functionally identical across frameworks. Thus, DynamicLRP's operation-level logic is portable to any system using dynamic computation graphs.


\section{Conclusion}
We have introduced a new algorithm (DynamicLRP) grounded in operations at the primitive tensor level, and a novel mathematical framework (the Promise System) that enables the extension of LRP to arbitrary model architectures. Together, these innovations represent a framework to practically enable the use of LRP for model developers. Looking into existing XAI literature and making use of some theory, we give a sound justification for the choice of this method as a scalable and accurate method for XAI. Empirically, we demonstrate state-of-the-art performance of DynamicLRP both in terms of model attribution faithfulness and in terms of computational overhead. We believe that while further experimental work could be of benefit, the intellectual contributions of this work are substantial and form the groundwork for a promising approach to XAI.

\section*{Acknowledgements}
The authors thank Julie Walsh, Dan Shonfeld, and Mirela Gondor from Scotiabank for their support and encouragement throughout this project. This work was made possible in part by a generous compute donation from Lambda Labs to PMA. 

\bibliography{iclr2026} 
\bibliographystyle{iclr2026/iclr2026_conference}

\appendix
\renewcommand{\thetable}{A.\arabic{table}} 
\setcounter{table}{0} 

\renewcommand{\thefigure}{A.\arabic{figure}} 
\setcounter{figure}{0} 

\section{LRP Derivation via Deep Taylor Decomposition}
\label{appendix:taylor}

LRP redistributes the prediction score (relevance) backwards through a neural network. The challenge of faithful attribution lies in identifying a meaningful distribution rule for the relevance assigned to input $i$ by output neuron $j$, $R_{i\leftarrow j}$. 
We take advantage of the Deep Taylor Decomposition framework~\cite{montavon2017explaining, achtibat2024attnlrp} to locally linearize and decompose operations into independent contributions, specifically by computing a first-order Taylor expansion of a function $f_j(x)$ at a reference point $x_0$:
\begin{align}
      f_j(x) &= f_j(x_0) + \sum_i \frac{\partial f_j}{\partial x_i}\Big|_{x_0} (x_i - x_{0,i})  + \mathcal{O}(|x -x_0|^2) \nonumber \\
      &= \sum_i \frac{\partial f_j}{\partial x_i}\Big|_{x_0} x_i + \tilde{b}_j
     \label{eq:taylor}
\end{align} 
where $\tilde{b}_j$ collects the constant terms and approximation error.
Assuming the relevance $R_j^l$ is proportional to the function activation $f_j(x)$, we can multiply each term in the expansion by a constant ratio $\frac{R_j^l}{f_j(x)}$ to obtain a recursion for relevance:

\begin{equation*}
     R_j^l = \sum_i \underbrace{\frac{\partial f_j}{\partial x_i}\Big|_{x_0} x_i \frac{R_j^l}{f_j(x)}}_{R_{i\leftarrow j}} + \underbrace{\tilde{b}_j \frac{R_j^l}{f_j(x)}}_{R_{b\leftarrow j}}
\end{equation*}

\section{Algorithms}

\iftoggle{arxiv}{}{\begin{algorithm}[!h]
\small
  \caption{Topological Traversal of Backward Graph}
  \label{alg:topological-traversal}
  \begin{algorithmic}[1]
  \State \textbf{Input:} Backward Graph $G = (V, E)$
  \State Initialize $S$ as a stack containing all nodes with $indegree(v) = 0$
  \State Initialize $pending(v) = indegree(v) \quad \forall v \in V$
  \While{$S$ is not empty}
    \State Pop node $v$ from $S$
    \State \Call{Process}{$v$} \Comment{Apply operation-specific propagation}
    \For{$w \in outadj(v)$}
      \State $pending(w) \gets pending(w) - 1$
      \If{$pending(w) == 0$}
        \State Push $w$ onto $S$
      \EndIf
    \EndFor
  \EndWhile
  \end{algorithmic}
\end{algorithm}
}

\begin{algorithm}[!h]
    \small
  \caption{Computation Graph Pre-Processing}
  \label{alg:graph-construction}
  \begin{algorithmic}[1]
    \State \textbf{Input:} Model output \texttt{hidden\_states}
    \State Initialize $\texttt{in\_adj} \gets \emptyset$, $\texttt{out\_adj} \gets \emptyset$
    \State Initialize $\texttt{visited} \gets \emptyset$, $\texttt{topo\_stack} \gets \emptyset$
    \State $\texttt{root} \gets \texttt{hidden\_states.grad\_fn}$
    \State \Call{DFS}{\texttt{root, in\_adj, out\_adj, visited, topo\_stack}}
    \State \Return \texttt{in\_adj, out\_adj, topo\_stack} \Comment{In-adjacency list, Out-adjacency list, Topologically sorted nodes}
    \\
    \Function{DFS}{\texttt{node, in\_adj, out\_adj, visited, topo\_stack}}
      \If{\texttt{node} = \text{None} or \texttt{node} in \texttt{visited}} \Comment{Skip processed}
        \State \Return
      \EndIf
      \State \texttt{visited.add(node)} \Comment{Mark visited}
      \For{each \texttt{child} in \texttt{node.next\_functions}}
        \State \texttt{out\_adj[node].append(child)} \Comment{Record dependencies}
        \State \texttt{in\_adj[child].append(node)}
        \State \Call{DFS}{\texttt{child, in\_adj, out\_adj, visited, topo\_stack}} \Comment{Recurse}
      \EndFor
      \State \texttt{topo\_stack.push(node)} \Comment{Add to sort}
    \EndFunction
  \end{algorithmic}
\end{algorithm}

We extend the standard propagation functions to handle Promise inputs. Depending on whether the current node is an Arg Node, the function either records the chain or resolves and completes the Promise.

\begin{algorithm}[!h]
    \small
  \caption{Non-Arg Node Propagation Function Promise Handling}
  \label{alg:non-argnode-propagation}
  \begin{algorithmic}[1]
    \State \textbf{Input:} Autograd Node \texttt{node}, Propagation input $R_{out}$
    \State \textbf{Output:} Propagation output list $R_{in}$

    \If{$R_{out}$ is a Promise Branch} \Comment{Record chain if exploring}
      \State Define \texttt{fwd} as a closure of \texttt{node}'s forward pass.
      \State Define \texttt{bwd} as a closure of \texttt{node}'s relevance distribution logic.
      \State $R_{out}\texttt{.record(fwd, bwd)}$
      \State \Return $R_{out}$
    \EndIf

    \State \textit{// Propagate $R_{out}$...} \Comment{Standard LRP}
    \State \Return $R_{in}$
  \end{algorithmic}
\end{algorithm}

\begin{algorithm}[!h]
    \small
  \caption{Arg Node Propagation Function Promise Handling}
  \label{alg:argnode-propagation}
  \begin{algorithmic}[1]
    \State \textbf{Input:} Autograd Node \texttt{node}, Propagation input $R_{out}$
    \State \textbf{Output:} Propagation output list $R_{in}$

    \If{$R_{out}$ is a Promise Branch} \Comment{Retrieve forward activation}
      \State Define \texttt{retrieve\_fwd\_output} as a function that extracts \texttt{node}'s forward pass output given its saved tensors.
      \State \texttt{activation = retrieve\_fwd\_output(node)}
      \State $R_{out}\texttt{.set\_arg(activation)}$
      \State $R_{out}\texttt{.trigger\_promise\_completion()}$
      \If{$R_{out}\texttt{.promise\_is\_complete}$} \Comment{Resume propagation}
        \State $R_{out} = R_{out}\texttt{.propagated\_relevance}$
      \Else
        \State \textit{// Signal for queueing this Node until the Promise is complete}
        \State \Return
      \EndIf
    \EndIf

    \State \textit{// Propagate $R_{out}$...}
    \State \Return $R_{in}$
  \end{algorithmic}
\end{algorithm}

\newpage
\textbf{Explicit Steps in the Algorithm}
\begin{enumerate}
  \item \textbf{Forward Pass:} Get the model output and computation graph built by the auto-differentiation library.
  \item \textbf{Graph Augmentation Pass:} Attach any additional required metadata for LRP to each node in the computation graph. This includes in-neighbours, topological ordering, and computation flags. See Algorithm~\ref{alg:graph-construction}.
  \item \textbf{Backward Pass:} Starting from the output, traverse the graph using topological traversal. A node is traversed when all output relevance is received, or when it is part of a Promise's resolution path, and marked as propagated when it distributes relevance to its inputs.
  \item \textbf{Relevance Aggregation:} Input nodes collect relevance from all incoming promises, yielding the final attribution scores.
\end{enumerate}

\textbf{Readiness and Completeness}
\begin{itemize}
  \item \textbf{Readiness:} A promise is ready when all output relevance (from downstream nodes) has been received.
  \item \textbf{Completeness:} A promise is complete when it has distributed its relevance to all inputs according to its propagation rule.
\end{itemize}

This mechanism ensures that relevance is propagated in accordance with the mathematical definition of LRP, maintaining conservation and proportionality at each node.

\section{Promise Examples}
\subsection{Basic Promise Mechanism}
\label{appendix:basic-promise-example}

In this section, we will provide a complete walk-through for the example provided in Section~\ref{sec:promise-system}. For each node $v$, consider $fwd_v$ to be a closure of the node's forward operation, and $bwd_v$ to be a closure of
its backward relevance propagation function. Let $y_E$ be the forward output at node E. Then, we can reconstruct the forward input at node A, $x_A$, by computing $x_A = fwd_B(fwd_C(fwd_D(y_E)))$. Let $R_A$ be the accumulated relevance at A from its in-neighbours. When we have recomputed $x_A$, we can then compute
$R_{B \leftarrow A} = bwd_A(R_A, x_A)$. And similarly, we compute $R_{E \leftarrow D} = bwd_D(bwd_C(bwd_B(R_{B \leftarrow A})))$.

We accomplish this during traversal with the use of Promises, enabling a lazy evaluation of these values. We colour nodes red upon traversing them. We also indicate the location of the traversal pointer with $\uparrow$.

\begin{center}
\resizebox{0.5\linewidth}{!}{%
\begin{tikzpicture}
  \node[node, red, label=below:{$\uparrow$}] (node1) {A};
  \node[node] (node2) [right=of node1] {B};
  \node[node] (node3) [right=of node2] {C};
  \node[node] (node4) [right=of node3] {D};
  \node[node, orange] (node5) [right=of node4] {E};

  \draw[->] (node1.east) -- (node2.west);
  \draw[->] (node2.east) -- (node3.west);
  \draw[->] (node3.east) -- (node4.west);
  \draw[->] (node4.east) -- (node5.west);
\end{tikzpicture}
}
\end{center}

We encounter node A, which we find requires a Promise. We instantiate the Promise metadata object $P$ and propagate its Branch $b_P$ to its out-neighbour, and continue traversal.

\begin{center}
\resizebox{0.5\linewidth}{!}{%
\begin{tikzpicture}
  \node[node, red] (node1) {A};
  \node[node, red] (node2) [right=of node1] {B};
  \node[align=center,anchor=north] (lab) at (node2.south) {$\uparrow$\\$b_P$ stores $fwd_B$, $bwd_B$};
  \node[node] (node3) [right=of node2] {C};
  \node[node] (node4) [right=of node3] {D};
  \node[node, orange] (node5) [right=of node4] {E};

  \draw[->] (node1.east) -- (node2.west) node[midway, above] {$b_P$};
  \draw[->] (node2.east) -- (node3.west);
  \draw[->] (node3.east) -- (node4.west);
  \draw[->] (node4.east) -- (node5.west);
\end{tikzpicture}}
\end{center}

Upon reaching node B, since it is not an Arg Node, $b_P$ will record $fwd_B$ and $bwd_B$, then will be passed on to B's out-neighbour.
It will continue and do this for C and D.

\begin{center}
\resizebox{0.5\linewidth}{!}{%
\begin{tikzpicture}
  \node[node, red] (node1) {A};
  \node[node, red] (node2) [right=of node1] {B};
  \node[node, red] (node3) [right=of node2] {C};
  \node[align=center,anchor=north] (lab) at (node3.south) {$\uparrow$\\$b_P$ stores $fwd_C$, $bwd_C$};
  \node[node] (node4) [right=of node3] {D};
  \node[node, orange] (node5) [right=of node4] {E};

  \draw[->] (node1.east) -- (node2.west) node[midway, above] {$b_P$};
  \draw[->] (node2.east) -- (node3.west) node[midway, above] {$b_P$};
  \draw[->] (node3.east) -- (node4.west);
  \draw[->] (node4.east) -- (node5.west);
\end{tikzpicture}
}
\end{center}

\begin{center}
\resizebox{0.5\linewidth}{!}{%
\begin{tikzpicture}
  \node[node, red] (node1) {A};
  \node[node, red] (node2) [right=of node1] {B};
  \node[node, red] (node3) [right=of node2] {C};
  \node[node, red] (node4) [right=of node3] {D};
  \node[align=center,anchor=north] (lab) at (node4.south) {$\uparrow$\\$b_P$ stores $fwd_D$, $bwd_D$};
  \node[node, orange] (node5) [right=of node4] {E};

  \draw[->] (node1.east) -- (node2.west) node[midway, above] {$b_P$};
  \draw[->] (node2.east) -- (node3.west) node[midway, above] {$b_P$};
  \draw[->] (node3.east) -- (node4.west) node[midway, above] {$b_P$};
  \draw[->] (node4.east) -- (node5.west);
\end{tikzpicture}
}
\end{center}

Until, finally, $b_P$ encounters node E, which is an Arg Node, and we are able to retrieve the forward pass output of E, $y_E$.

\begin{center}
\resizebox{0.5\linewidth}{!}{%
\begin{tikzpicture}
  \node[node, red] (node1) {A};
  \node[node, red] (node2) [right=of node1] {B};
  \node[node, red] (node3) [right=of node2] {C};
  \node[node, red] (node4) [right=of node3] {D};
  \node[node, orange] (node5) [right=of node4] {E};
  \node[align=center,anchor=north] (lab) at (node5.south) {$\uparrow$\\retrieve $y_E$};

  \draw[->] (node1.east) -- (node2.west) node[midway, above] {$b_P$};
  \draw[->] (node2.east) -- (node3.west) node[midway, above] {$b_P$};
  \draw[->] (node3.east) -- (node4.west) node[midway, above] {$b_P$};
  \draw[->] (node4.east) -- (node5.west) node[midway, above] {$b_P$};
\end{tikzpicture}
}
\end{center}

We then:
\begin{itemize}
  \item Use the stored $fwd$ closures from B, C, D to compute $x_A = fwd_B(fwd_C(fwd_D(y_E)))$, and store it within the Promise metadata object $P$.
  \item Check in $P$ if the Promise is now in Ready state (all branches have found an Arg Node), in this case there is only the one branch.
  \item If Ready, and the Promise has no incomplete parent Promises (explained the next section), compute $R_{E \leftarrow A} = bwd_D(bwd_C(bwd_B(bwd_A(R_A, x_A))))$.
\end{itemize}

\begin{center}
\resizebox{0.65\linewidth}{!}{%
\begin{tikzpicture}
  \node[node, red] (node1) {A};
  \node[node, red] (node2) [right=of node1] {B};
  \node[node, red] (node3) [right=of node2] {C};
  \node[node, red] (node4) [right=of node3] {D};
  \node[node, orange] (node5) [right=of node4] {E};
  \node[align=center,anchor=north] (lab) at (node5.south) {$\uparrow$\\retrieve $y_E$};
  \node[align=center,anchor=north, yshift=-1.1cm, xshift=-1.3cm] (label1) at (node1.south) {Store $x_A$ in $P$};

  \draw[->] (node1.east) -- (node2.west) node[midway, above] {$b_P$};
  \draw[->] (node2.east) -- (node3.west) node[midway, above] {$b_P$};
  \draw[->] (node3.east) -- (node4.west) node[midway, above] {$b_P$};
  \draw[->] (node4.east) -- (node5.west) node[midway, above] {$b_P$};
  \begin{scope}[transform canvas={yshift=-1.8cm}]
    \draw[->, shorten <= -0.5cm, shorten >= -0.5cm] (node5.north) -- (node1.north) node[midway, above] {Compute $x_A$};
  \end{scope}
  \begin{scope}[transform canvas={yshift=-2.4cm}]
    \draw[->, shorten <= -0.5cm, shorten >= -0.5cm] (node1.north) -- (node5.north) node[midway, below] {Compute $R_{E \leftarrow A}$};
  \end{scope}
\end{tikzpicture}
}
\end{center}

And now, we have lazily computed the propagations of the path starting from A and ending at E, while keeping the traversal pointer from
backtracking any previously visited nodes. Blue indicates relevance propagation is complete at that node.

\begin{center}
\resizebox{0.65\linewidth}{!}{%
\begin{tikzpicture}
  \node[node, blue] (node1) {A};
  \node[node, blue] (node2) [right=of node1] {B};
  \node[node, blue] (node3) [right=of node2] {C};
  \node[node, blue] (node4) [right=of node3] {D};
  \node[node, orange] (node5) [right=of node4] {E};
  \node[align=center,anchor=north] (lab) at (node5.south) {$\uparrow$\\Continue propagation with $R_{E \leftarrow A}$};

  \draw[->] (node1.east) -- (node2.west);
  \draw[->] (node2.east) -- (node3.west);
  \draw[->] (node3.east) -- (node4.west);
  \draw[->] (node4.east) -- (node5.west);
\end{tikzpicture}
}
\end{center}

\subsection{Promise Trees and Resolution}
\label{appendix:promise-trees}
A Promise Tree is created when a promise-generating operation receives a Promise as relevance input. The new Promise specific to that operation is created and linked to the input Promise with parent-child connections.

\begin{center}
\begin{tikzpicture}
  \node[node, purple] (node1) {A};
  \node[node, purple] (node2) [below left=of node1] {B};
  \node[node, orange] (node3) [below right=of node1] {C};
  \node[node, orange] (node4) [below left=of node2, xshift=0.6cm] {D};
  \node[node, orange] (node5) [below right=of node2, xshift=-0.6cm] {E};

  \draw[dashed, ->] (node1.south west) -- (node2.north east);
  \draw[dashed, ->] (node1.south east) -- (node3.north west);
  \draw[dashed, ->] (node2.south west) -- (node4.north);
  \draw[dashed, ->] (node2.south east) -- (node5.north);
\end{tikzpicture}
\end{center}

In this example, A and B are promise-generating operations, marked purple, and D, E, C are Arg Nodes, marked orange. We also use "->" to indicate the node at which the traversal pointer is pointing at each step. We mark nodes as red when we have traversed them, blue when we have propagated relevance through them.
It is important to note the difference between these two states due to the delaying of propagation through Promises.
Node A is traversed first, and propagates Promise Branches $b_{A,1}$ and $b_{A,2}$ through to its out-neighbours.

\begin{minipage}{.5\textwidth}
\begin{center}

\vspace{0.3cm}
Computation Graph

\vspace{0.3cm}
\begin{tikzpicture}
  \node[node, red, label=left:{->}] (node1) {A};
  \node[node, purple] (node2) [below left=of node1] {B};
  \node[node, orange] (node3) [below right=of node1] {C};
  \node[node, orange] (node4) [below left=of node2, xshift=0.6cm] {D};
  \node[node, orange] (node5) [below right=of node2, xshift=-0.6cm] {E};

  \draw[dashed, ->] (node1.south west) -- (node2.north east) node[midway, above, xshift=-0.1cm] {$b_{A,1}$};
  \draw[dashed, ->] (node1.south east) -- (node3.north west) node[midway, above, xshift=0.1cm] {$b_{A,2}$};
  \draw[dashed, ->] (node2.south west) -- (node4.north);
  \draw[dashed, ->] (node2.south east) -- (node5.north);

\end{tikzpicture}
\end{center}
\end{minipage}
\begin{minipage}{.5\textwidth}
\begin{center}

Promise Tree

\vspace{0.3cm}
\resizebox{0.33\linewidth}{!}{%
\begin{tikzpicture}
  \node[diamondnode, purple] (node1) {$P_A$};
  \node[node] (node2) [yshift=-1.0cm, xshift=-0.6cm] {$b_{A,1}$};
  \node[node] (node3) [yshift=-1.0cm, xshift=0.6cm] {$b_{A,2}$};

  \draw (node1.south west) -- (node2.north);
  \draw (node1.south east) -- (node3.north);
\end{tikzpicture}
}
\end{center}
\end{minipage}

We traverse to node B now, and find that it receives Promise Branch $b_{A,1}$. But B is a promise-generating operation, so a new Promise $P_B$ is created.
$P_B$ is added as a child of the branch $b_{A,1}$ and $b_{A,1}$ is added as a parent to $P_B$, creating a Promise Tree. B then propagates the branches of $P_B$, $b_{B,1}$ and $b_{B,2}$ to its out-neighbours.

\begin{minipage}{.5\textwidth}
\begin{center}

\vspace{0.3cm}
Computation Graph

\vspace{0.3cm}
\begin{tikzpicture}
  \node[node, red] (node1) {A};
  \node[node, red, label=left:{->}] (node2) [below left=of node1] {B};
  \node[node, orange] (node3) [below right=of node1] {C};
  \node[node, orange] (node4) [below left=of node2, xshift=0.6cm] {D};
  \node[node, orange] (node5) [below right=of node2, xshift=-0.6cm] {E};

  \draw[dashed, ->] (node1.south west) -- (node2.north east) node[midway, above, xshift=-0.1cm] {$b_{A,1}$};
  \draw[dashed, ->] (node1.south east) -- (node3.north west) node[midway, above, xshift=0.1cm] {$b_{A,2}$};
  \draw[dashed, ->] (node2.south west) -- (node4.north) node[midway, above, xshift=-0.1cm] {$b_{B,1}$};
  \draw[dashed, ->] (node2.south east) -- (node5.north) node[midway, above, xshift=0.1cm] {$b_{B,2}$};

\end{tikzpicture}
\end{center}
\end{minipage}
\begin{minipage}{.5\textwidth}
\begin{center}

\vspace{0.3cm}
Promise Tree

\vspace{0.3cm}
\resizebox{0.33\linewidth}{!}{%
\begin{tikzpicture}
  \node[diamondnode, purple] (node1) {$P_A$};
  \node[node] (node2) [yshift=-1.0cm, xshift=-0.6cm] {$b_{A,1}$};
  \node[node] (node3) [yshift=-1.0cm, xshift=0.6cm] {$b_{A,2}$};
  \node[diamondnode, purple] (node4) [below=of node2, yshift=0.3cm] {$P_B$};
  \node[node] (node5) [below left=of node4, yshift=0.62cm, xshift=1.0cm] {$b_{B,1}$};
  \node[node] (node6) [below right=of node4, yshift=0.62cm, xshift=-1.0cm] {$b_{B,2}$};

  \draw (node1.south west) -- (node2.north);
  \draw (node1.south east) -- (node3.north);
  \draw[<->] (node2.south) -- (node4.north);
  \draw (node4.south west) -- (node5.north);
  \draw (node4.south east) -- (node6.north);
\end{tikzpicture}
}
\end{center}
\end{minipage}

Assume that D, E, C are the first Arg Nodes encountered by branches $b_{B,1}$, $b_{B,2}$, and $b_{A,2}$, respectively, and that they are traversed in such order.
Traversing D causes $b_{B,1}$ to forward chain the activation upwards. We colour Promise Tree objects teal when forward chaining has occurred through them.
Since $P_B$ is still not in Complete state after this step, we enqueue D and stall relevance propagation along its path.

\begin{minipage}{.5\textwidth}
\begin{center}

\vspace{0.3cm}
Computation Graph

\vspace{0.3cm}
\begin{tikzpicture}
  \node[node, red] (node1) {A};
  \node[node, red] (node2) [below left=of node1] {B};
  \node[node, orange] (node3) [below right=of node1] {C};
  \node[node, red, label=below:{\small stalled}, label=left:{->}] (node4) [below left=of node2, xshift=0.6cm] {D};
  \node[node, orange] (node5) [below right=of node2, xshift=-0.6cm] {E};

  \draw[dashed, ->] (node1.south west) -- (node2.north east) node[midway, above, xshift=-0.1cm] {$b_{A,1}$};
  \draw[dashed, ->] (node1.south east) -- (node3.north west) node[midway, above, xshift=0.1cm] {$b_{A,2}$};
  \draw[dashed, ->] (node2.south west) -- (node4.north) node[midway, above, xshift=-0.1cm] {$b_{B,1}$};
  \draw[dashed, ->] (node2.south east) -- (node5.north) node[midway, above, xshift=0.1cm] {$b_{B,2}$};

\end{tikzpicture}
\end{center}
\end{minipage}
\begin{minipage}{.5\textwidth}
\begin{center}

\vspace{0.3cm}
Promise Tree

\vspace{0.3cm}
\resizebox{0.33\linewidth}{!}{%
\begin{tikzpicture}
  \node[diamondnode, purple] (node1) {$P_A$};
  \node[node] (node2) [yshift=-1.0cm, xshift=-0.6cm] {$b_{A,1}$};
  \node[node] (node3) [yshift=-1.0cm, xshift=0.6cm] {$b_{A,2}$};
  \node[diamondnode, purple] (node4) [below=of node2, yshift=0.3cm] {$P_B$};
  \node[node, teal] (node5) [below left=of node4, yshift=0.62cm, xshift=1.0cm] {$b_{B,1}$};
  \node[node] (node6) [below right=of node4, yshift=0.62cm, xshift=-1.0cm] {$b_{B,2}$};

  \draw (node1.south west) -- (node2.north);
  \draw (node1.south east) -- (node3.north);
  \draw[<->] (node2.south) -- (node4.north);
  \draw (node4.south west) -- (node5.north);
  \draw (node4.south east) -- (node6.north);
\end{tikzpicture}
}
\end{center}
\end{minipage}

Likewise, once E is traversed, it will provide its forward output to $b_{B,2}$ to forward chain upwards. Since all branches of $P_B$ will have
forwarded their arguments, $P_B$ is now in Ready state.

\begin{minipage}{.5\textwidth}
\begin{center}

\vspace{0.3cm}
Computation Graph

\vspace{0.3cm}
\begin{tikzpicture}
  \node[node, red] (node1) {A};
  \node[node, red] (node2) [below left=of node1] {B};
  \node[node, orange] (node3) [below right=of node1] {C};
  \node[node, red, label=below:{\small stalled}] (node4) [below left=of node2, xshift=0.6cm] {D};
  \node[node, red, label=left:{->}] (node5) [below right=of node2, xshift=-0.6cm] {E};

  \draw[dashed, ->] (node1.south west) -- (node2.north east) node[midway, above, xshift=-0.1cm] {$b_{A,1}$};
  \draw[dashed, ->] (node1.south east) -- (node3.north west) node[midway, above, xshift=0.1cm] {$b_{A,2}$};
  \draw[dashed, ->] (node2.south west) -- (node4.north) node[midway, above, xshift=-0.1cm] {$b_{B,1}$};
  \draw[dashed, ->] (node2.south east) -- (node5.north) node[midway, above, xshift=0.1cm] {$b_{B,2}$};

\end{tikzpicture}
\end{center}
\end{minipage}
\begin{minipage}{.5\textwidth}
\begin{center}

Promise Tree

\vspace{0.3cm}
\resizebox{0.33\linewidth}{!}{%
\begin{tikzpicture}
  \node[diamondnode, purple] (node1) {$P_A$};
  \node[node] (node2) [yshift=-1.0cm, xshift=-0.6cm] {$b_{A,1}$};
  \node[node] (node3) [yshift=-1.0cm, xshift=0.6cm] {$b_{A,2}$};
  \node[diamondnode, teal] (node4) [below=of node2, yshift=0.3cm] {$P_B$};
  \node[node, teal] (node5) [below left=of node4, yshift=0.62cm, xshift=1.0cm] {$b_{B,1}$};
  \node[node, teal] (node6) [below right=of node4, yshift=0.62cm, xshift=-1.0cm] {$b_{B,2}$};

  \draw (node1.south west) -- (node2.north);
  \draw (node1.south east) -- (node3.north);
  \draw[<->] (node2.south) -- (node4.north);
  \draw (node4.south west) -- (node5.north);
  \draw (node4.south east) -- (node6.north);
\end{tikzpicture}
}
\end{center}
\end{minipage}

When a Promise becomes Ready, this triggers the first half of the Promise Tree resolution at that Promise. It will apply its characteristic
operation using the results materialized by its branches' forward chains as inputs, and pass this output to its parents in the Promise Tree by
calling \texttt{parent.setarg(self.op\_result)}.

However, after this step, $P_B$ is still not in Complete state, so we must also enqueue E and stall relevance propagation along this path.

\begin{minipage}{.5\textwidth}
\begin{center}

\vspace{0.3cm}
Computation Graph

\vspace{0.3cm}
\begin{tikzpicture}
  \node[node, red] (node1) {A};
  \node[node, red] (node2) [below left=of node1] {B};
  \node[node, orange] (node3) [below right=of node1] {C};
  \node[node, red, label=below:{\small stalled}] (node4) [below left=of node2, xshift=0.6cm] {D};
  \node[node, red, label=below:{\small stalled}] (node5) [below right=of node2, xshift=-0.6cm] {E};

  \draw[dashed, ->] (node1.south west) -- (node2.north east) node[midway, above, xshift=-0.1cm] {$b_{A,1}$};
  \draw[dashed, ->] (node1.south east) -- (node3.north west) node[midway, above, xshift=0.1cm] {$b_{A,2}$};
  \draw[dashed, ->] (node2.south west) -- (node4.north) node[midway, above, xshift=-0.1cm] {$b_{B,1}$};
  \draw[dashed, ->] (node2.south east) -- (node5.north) node[midway, above, xshift=0.1cm] {$b_{B,2}$};

\end{tikzpicture}
\end{center}
\end{minipage}
\begin{minipage}{.5\textwidth}
\begin{center}

Promise Tree

\vspace{0.3cm}
\resizebox{0.33\linewidth}{!}{%
\begin{tikzpicture}
  \node[diamondnode, purple] (node1) {$P_A$};
  \node[node, teal] (node2) [yshift=-1.0cm, xshift=-0.6cm] {$b_{A,1}$};
  \node[node] (node3) [yshift=-1.0cm, xshift=0.6cm] {$b_{A,2}$};
  \node[diamondnode, teal] (node4) [below=of node2, yshift=0.3cm] {$P_B$};
  \node[node, teal] (node5) [below left=of node4, yshift=0.62cm, xshift=1.0cm] {$b_{B,1}$};
  \node[node, teal] (node6) [below right=of node4, yshift=0.62cm, xshift=-1.0cm] {$b_{B,2}$};

  \draw (node1.south west) -- (node2.north);
  \draw (node1.south east) -- (node3.north);
  \draw[<->] (node2.south) -- (node4.north);
  \draw (node4.south west) -- (node5.north);
  \draw (node4.south east) -- (node6.north);
\end{tikzpicture}
}
\end{center}
\end{minipage}

Now, when C is traversed, the pattern repeats. $b_{A,2}$ will forward-chain the activation from C, and it will cause $P_A$ to reach Ready state.

\begin{minipage}{.5\textwidth}
\begin{center}

\vspace{0.3cm}
Computation Graph

\vspace{0.3cm}
\begin{tikzpicture}
  \node[node, red] (node1) {A};
  \node[node, red] (node2) [below left=of node1] {B};
  \node[node, red, label=left:{->}] (node3) [below right=of node1] {C};
  \node[node, red, label=below:{\small stalled}] (node4) [below left=of node2, xshift=0.6cm] {D};
  \node[node, red, label=below:{\small stalled}] (node5) [below right=of node2, xshift=-0.6cm] {E};

  \draw[dashed, ->] (node1.south west) -- (node2.north east) node[midway, above, xshift=-0.1cm] {$b_{A,1}$};
  \draw[dashed, ->] (node1.south east) -- (node3.north west) node[midway, above, xshift=0.1cm] {$b_{A,2}$};
  \draw[dashed, ->] (node2.south west) -- (node4.north) node[midway, above, xshift=-0.1cm] {$b_{B,1}$};
  \draw[dashed, ->] (node2.south east) -- (node5.north) node[midway, above, xshift=0.1cm] {$b_{B,2}$};

\end{tikzpicture}
\end{center}
\end{minipage}
\begin{minipage}{.5\textwidth}
\begin{center}

Promise Tree

\vspace{0.3cm}
\resizebox{0.33\linewidth}{!}{%
\begin{tikzpicture}
  \node[diamondnode, teal] (node1) {$P_A$};
  \node[node, teal] (node2) [yshift=-1.0cm, xshift=-0.6cm] {$b_{A,1}$};
  \node[node, teal] (node3) [yshift=-1.0cm, xshift=0.6cm] {$b_{A,2}$};
  \node[diamondnode, teal] (node4) [below=of node2, yshift=0.3cm] {$P_B$};
  \node[node, teal] (node5) [below left=of node4, yshift=0.62cm, xshift=1.0cm] {$b_{B,1}$};
  \node[node, teal] (node6) [below right=of node4, yshift=0.62cm, xshift=-1.0cm] {$b_{B,2}$};

  \draw (node1.south west) -- (node2.north);
  \draw (node1.south east) -- (node3.north);
  \draw[<->] (node2.south) -- (node4.north);
  \draw (node4.south west) -- (node5.north);
  \draw (node4.south east) -- (node6.north);
\end{tikzpicture}
}
\end{center}
\end{minipage}

And since $P_A$ has no parents in the Promise Tree, it will now begin the backpropagation of relevance, now that all arguments have been resolved.
It will do this for all of its own branches using the backward chains stored within them, and recursively do so for the children of those branches.

We also colour nodes in the Promise Tree as blue when relevance has been propagated through them.

Since this all occurs in one step, we label the order in which propagation occurs at each node in both the Computation Graph and Promise Tree.

\begin{minipage}{.5\textwidth}
\begin{center}

\vspace{0.3cm}
Computation Graph

\vspace{0.3cm}
\begin{tikzpicture}
  \node[node, blue, label={\small 1}] (node1) {A};
  \node[node, blue, label=left:{\small 3}] (node2) [below left=of node1] {B};
  \node[node, red, label=left:{->}] (node3) [below right=of node1] {C};
  \node[node, red] (node4) [below left=of node2, xshift=0.6cm] {D};
  \node[node, red] (node5) [below right=of node2, xshift=-0.6cm] {E};

  \draw[dashed, ->] (node1.south west) -- (node2.north east) node[midway, above, xshift=-0.1cm] {$b_{A,1}$};
  \draw[dashed, ->] (node1.south east) -- (node3.north west) node[midway, above, xshift=0.1cm] {$b_{A,2}$};
  \draw[dashed, ->] (node2.south west) -- (node4.north) node[midway, above, xshift=-0.1cm] {$b_{B,1}$};
  \draw[dashed, ->] (node2.south east) -- (node5.north) node[midway, above, xshift=0.1cm] {$b_{B,2}$};

\end{tikzpicture}
\end{center}
\end{minipage}
\begin{minipage}{.5\textwidth}
\begin{center}

Promise Tree

\vspace{0.3cm}
\resizebox{0.40\linewidth}{!}{%
\begin{tikzpicture}
  \node[diamondnode, blue, label={\small 1}] (node1) {$P_A$};
  \node[node, blue, label=left:{\small 2}] (node2) [yshift=-1.0cm, xshift=-0.6cm] {$b_{A,1}$};
  \node[node, blue, label=right:{\small 6}] (node3) [yshift=-1.0cm, xshift=0.6cm] {$b_{A,2}$};
  \node[diamondnode, blue, label=left:{\small 3}] (node4) [below=of node2, yshift=0.3cm] {$P_B$};
  \node[node, blue, label=left:{\small 4}] (node5) [below left=of node4, yshift=0.62cm, xshift=1.0cm] {$b_{B,1}$};
  \node[node, blue, label=right:{\small 5}] (node6) [below right=of node4, yshift=0.62cm, xshift=-1.0cm] {$b_{B,2}$};

  \draw (node1.south west) -- (node2.north);
  \draw (node1.south east) -- (node3.north);
  \draw[<->] (node2.south) -- (node4.north);
  \draw (node4.south west) -- (node5.north);
  \draw (node4.south east) -- (node6.north);
\end{tikzpicture}
}
\end{center}
\end{minipage}

Note that although the relevance has propagated through the Promise Branches $b_{A,2}$, $b_{B,1}$, and $b_{B,2}$, the propagation functions
of the nodes at the end of their chains, D, E, C, respectively, have not been called at this moment, thus we do not say that relevance
has been propagated through them yet. Their propagation will resume in further iterations using the relevances passed to them through
those branches.

And see that the traversal pointer has remained at C, so we simply continue propagating along its path.
Also, since $P_B$ is Complete, D and E will also be dequeued in the coming iterations to continue propagation.

\begin{minipage}{.5\textwidth}
\begin{center}

\vspace{0.3cm}
Computation Graph

\vspace{0.3cm}
\begin{tikzpicture}
  \node[node, blue] (node1) {A};
  \node[node, blue] (node2) [below left=of node1] {B};
  \node[node, blue, label=left:{->}] (node3) [below right=of node1] {C};
  \node[node, red, label=below:{queued}] (node4) [below left=of node2, xshift=0.6cm] {D};
  \node[node, red, label=below:{queued}] (node5) [below right=of node2, xshift=-0.6cm] {E};

  \draw[dashed, ->] (node1.south west) -- (node2.north east) node[midway, above, xshift=-0.1cm] {$b_{A,1}$};
  \draw[dashed, ->] (node1.south east) -- (node3.north west) node[midway, above, xshift=0.1cm] {$b_{A,2}$};
  \draw[dashed, ->] (node2.south west) -- (node4.north) node[midway, above, xshift=-0.1cm] {$b_{B,1}$};
  \draw[dashed, ->] (node2.south east) -- (node5.north) node[midway, above, xshift=0.1cm] {$b_{B,2}$};

\end{tikzpicture}
\end{center}
\end{minipage}
\begin{minipage}{.5\textwidth}
\begin{center}

\vspace{0.3cm}
Promise Tree

\vspace{0.3cm}
\resizebox{0.33\linewidth}{!}{%
\begin{tikzpicture}
  \node[diamondnode, blue] (node1) {$P_A$};
  \node[node, blue] (node2) [yshift=-1.0cm, xshift=-0.6cm] {$b_{A,1}$};
  \node[node, blue] (node3) [yshift=-1.0cm, xshift=0.6cm] {$b_{A,2}$};
  \node[diamondnode, blue] (node4) [below=of node2, yshift=0.3cm] {$P_B$};
  \node[node, blue] (node5) [below left=of node4, yshift=0.62cm, xshift=1.0cm] {$b_{B,1}$};
  \node[node, blue] (node6) [below right=of node4, yshift=0.62cm, xshift=-1.0cm] {$b_{B,2}$};

  \draw (node1.south west) -- (node2.north);
  \draw (node1.south east) -- (node3.north);
  \draw[<->] (node2.south) -- (node4.north);
  \draw (node4.south west) -- (node5.north);
  \draw (node4.south east) -- (node6.north);
\end{tikzpicture}
}
\end{center}
\end{minipage}

\begin{minipage}{.5\textwidth}
\begin{center}

\begin{tikzpicture}
  \node[node, blue] (node1) {A};
  \node[node, blue] (node2) [below left=of node1] {B};
  \node[node, blue] (node3) [below right=of node1] {C};
  \node[node, blue, label=left:{->}] (node4) [below left=of node2, xshift=0.6cm] {D};
  \node[node, red] (node5) [below right=of node2, xshift=-0.6cm] {E};

  \draw[dashed, ->] (node1.south west) -- (node2.north east) node[midway, above, xshift=-0.1cm] {$b_{A,1}$};
  \draw[dashed, ->] (node1.south east) -- (node3.north west) node[midway, above, xshift=0.1cm] {$b_{A,2}$};
  \draw[dashed, ->] (node2.south west) -- (node4.north) node[midway, above, xshift=-0.1cm] {$b_{B,1}$};
  \draw[dashed, ->] (node2.south east) -- (node5.north) node[midway, above, xshift=0.1cm] {$b_{B,2}$};

\end{tikzpicture}
\end{center}
\end{minipage}
\begin{minipage}{.5\textwidth}
\begin{center}

\resizebox{0.33\linewidth}{!}{%
\begin{tikzpicture}
  \node[diamondnode, blue] (node1) {$P_A$};
  \node[node, blue] (node2) [yshift=-1.0cm, xshift=-0.6cm] {$b_{A,1}$};
  \node[node, blue] (node3) [yshift=-1.0cm, xshift=0.6cm] {$b_{A,2}$};
  \node[diamondnode, blue] (node4) [below=of node2, yshift=0.3cm] {$P_B$};
  \node[node, blue] (node5) [below left=of node4, yshift=0.62cm, xshift=1.0cm] {$b_{B,1}$};
  \node[node, blue] (node6) [below right=of node4, yshift=0.62cm, xshift=-1.0cm] {$b_{B,2}$};

  \draw (node1.south west) -- (node2.north);
  \draw (node1.south east) -- (node3.north);
  \draw[<->] (node2.south) -- (node4.north);
  \draw (node4.south west) -- (node5.north);
  \draw (node4.south east) -- (node6.north);
\end{tikzpicture}
}
\end{center}
\end{minipage}

\begin{minipage}{.5\textwidth}
\begin{center}

\begin{tikzpicture}
  \node[node, blue] (node1) {A};
  \node[node, blue] (node2) [below left=of node1] {B};
  \node[node, blue] (node3) [below right=of node1] {C};
  \node[node, blue] (node4) [below left=of node2, xshift=0.6cm] {D};
  \node[node, blue, label=left:{->}] (node5) [below right=of node2, xshift=-0.6cm] {E};

  \draw[dashed, ->] (node1.south west) -- (node2.north east) node[midway, above, xshift=-0.1cm] {$b_{A,1}$};
  \draw[dashed, ->] (node1.south east) -- (node3.north west) node[midway, above, xshift=0.1cm] {$b_{A,2}$};
  \draw[dashed, ->] (node2.south west) -- (node4.north) node[midway, above, xshift=-0.1cm] {$b_{B,1}$};
  \draw[dashed, ->] (node2.south east) -- (node5.north) node[midway, above, xshift=0.1cm] {$b_{B,2}$};

\end{tikzpicture}
\end{center}
\end{minipage}
\begin{minipage}{.5\textwidth}
\begin{center}

\resizebox{0.33\linewidth}{!}{%
\begin{tikzpicture}
  \node[diamondnode, blue] (node1) {$P_A$};
  \node[node, blue] (node2) [yshift=-1.0cm, xshift=-0.6cm] {$b_{A,1}$};
  \node[node, blue] (node3) [yshift=-1.0cm, xshift=0.6cm] {$b_{A,2}$};
  \node[diamondnode, blue] (node4) [below=of node2, yshift=0.3cm] {$P_B$};
  \node[node, blue] (node5) [below left=of node4, yshift=0.62cm, xshift=1.0cm] {$b_{B,1}$};
  \node[node, blue] (node6) [below right=of node4, yshift=0.62cm, xshift=-1.0cm] {$b_{B,2}$};

  \draw (node1.south west) -- (node2.north);
  \draw (node1.south east) -- (node3.north);
  \draw[<->] (node2.south) -- (node4.north);
  \draw (node4.south west) -- (node5.north);
  \draw (node4.south east) -- (node6.north);
\end{tikzpicture}
}
\end{center}
\end{minipage}

\subsection{Solution to Promise Deadlock: Pre-Promises}
\label{appendix:pre-promises}

The solution to Promise Deadlock is to allow Promise Branches to "reach ahead" past the traversal frontier to locate their Arg Nodes, without triggering full relevance backpropagation.

When a node $v$ receives a Promise Branch input $p_1$ but is still waiting for other inputs (i.e., landed inputs < in-degree), we instantiate a \textbf{Pre-Promise}, which only has one branch $p_2$, at $v$
(since it only has one branch, we simply refer to the Pre-Promise as its branch $p_2$). This Pre-Promise:

\begin{enumerate}
    \item Has a parent connection to $p_1$ (enabling forward-chaining of activations)
    \item Does NOT register as a child of $p_1$ (preventing backward-chaining of relevance)
    \item Continues traversal to find its Arg Node
\end{enumerate}

The parent-child asymmetry creates the desired behaviour: $p_2$ can reach its Arg Node, forward-chain the activation back through its parent connection to satisfy $p_1$'s requirements, but relevance propagation is blocked at $v$ until the traversal heuristic naturally revisits it (when all inputs have landed).

Once $v$'s in-degree is satisfied later in the traversal:
\begin{itemize}
    \item The Pre-Promise $p_2$ is "promoted" by establishing its child connection to $p_1$
    \item If $p_1$'s Promise is now complete (all args resolved), backward-chaining proceeds
    \item The traversal continues normally from $p_2$'s Arg Node after backward-chaining.
\end{itemize}

To illustrate the function of Pre-Promises, we will pick up where the example in Section 4.3.3 left off. We have that D is trying to propagate a Promise Branch to E, but E does not
have all its inputs yet. Therefore, we create a Pre-Promise $p_3$, and allow traversal to continue at E, now giving $p_3$ as E's input from D.

\begin{center}
\resizebox{0.40\linewidth}{!}{%
\begin{tikzpicture}
\node[node, red] (node1) {A};
\node[node, red, label={\small stalled}] (node2) [right=of node1] {B};
\node[node] (node3) [right=of node2] {C};
\node[node, red] (node4) [below=of node3, xshift=-8mm] {D};
\node[node, red] (node5) [right=of node3] {E};

\draw[->] (node1.east) -- (node2.west) node[midway, above] {$p_1$};
\draw[->] (node2.east) -- (node3.west);
\draw[->] (node3.east) -- (node5.west);
\draw[->] (node1.east) .. controls +(down:7mm) and +(left:7mm) .. (node4.west) node[midway, above] {$p_2$};
\draw[->] (node4.east) .. controls +(up:7mm) and +(left:7mm) .. (node5.west) node[midway, below] {$p_3$};
\end{tikzpicture}
}
\end{center}

$p_3$ retrieves and forward-chains the activation at E to $p_2$, which then forward-chains it to obtain A's missing activation. We now have that A, B, D, E have gone through
forward-chaining. We colour them as teal to signify this.

\begin{center}
\resizebox{0.40\linewidth}{!}{%
\begin{tikzpicture}
\node[node, teal] (node1) {A};
\node[node, teal, label={\small stalled}] (node2) [right=of node1] {B};
\node[node] (node3) [right=of node2] {C};
\node[node, teal] (node4) [below=of node3, xshift=-8mm] {D};
\node[node, teal] (node5) [right=of node3] {E};

\draw[->] (node1.east) -- (node2.west) node[midway, above] {$p_1$};
\draw[->] (node2.east) -- (node3.west);
\draw[->] (node3.east) -- (node5.west);
\draw[->] (node1.east) .. controls +(down:7mm) and +(left:7mm) .. (node4.west) node[midway, above] {$p_2$};
\draw[->] (node4.east) .. controls +(up:7mm) and +(left:7mm) .. (node5.west) node[midway, below] {$p_3$};
\end{tikzpicture}
}
\end{center}

This is followed by $P$ triggering relevance backpropagation through all of its branches. Blue nodes signify that true relevance values have been propagated through them.
Crucially, E does not continue this propagation due to the lack of a child connection from $p_2$ to $p_3$.

\begin{center}
\resizebox{0.4\linewidth}{!}{%
\begin{tikzpicture}
\node[node, blue] (node1) {A};
\node[node, blue] (node2) [right=of node1] {B};
\node[node] (node3) [right=of node2] {C};
\node[node, blue] (node4) [below=of node3, xshift=-8mm] {D};
\node[node, teal] (node5) [right=of node3] {E};

\draw[->] (node1.east) -- (node2.west) node[midway, above] {$p_1$};
\draw[->] (node2.east) -- (node3.west);
\draw[->] (node3.east) -- (node5.west);
\draw[->] (node1.east) .. controls +(down:7mm) and +(left:7mm) .. (node4.west) node[midway, above] {$p_2$};
\draw[->] (node4.east) .. controls +(up:7mm) and +(left:7mm) .. (node5.west) node[midway, below] {$p_3$};
\end{tikzpicture}
}
\end{center}

Now, C is able to be traversed.

\begin{center}
\resizebox{0.4\linewidth}{!}{%
\begin{tikzpicture}
\node[node, blue] (node1) {A};
\node[node, blue] (node2) [right=of node1] {B};
\node[node, red] (node3) [right=of node2] {C};
\node[node, blue] (node4) [below=of node3, xshift=-8mm] {D};
\node[node, teal] (node5) [right=of node3] {E};

\draw[->] (node1.east) -- (node2.west) node[midway, above] {$p_1$};
\draw[->] (node2.east) -- (node3.west);
\draw[->] (node3.east) -- (node5.west);
\draw[->] (node1.east) .. controls +(down:7mm) and +(left:7mm) .. (node4.west) node[midway, above] {$p_2$};
\draw[->] (node4.east) .. controls +(up:7mm) and +(left:7mm) .. (node5.west) node[midway, below] {$p_3$};
\end{tikzpicture}
}
\end{center}

And we propagate through it to E, finally.

\begin{center}
\resizebox{0.40\linewidth}{!}{%
\begin{tikzpicture}
\node[node, blue] (node1) {A};
\node[node, blue] (node2) [right=of node1] {B};
\node[node, blue] (node3) [right=of node2] {C};
\node[node, blue] (node4) [below=of node3, xshift=-8mm] {D};
\node[node, red] (node5) [right=of node3] {E};

\draw[->] (node1.east) -- (node2.west) node[midway, above] {$p_1$};
\draw[->] (node2.east) -- (node3.west);
\draw[->] (node3.east) -- (node5.west);
\draw[->] (node1.east) .. controls +(down:7mm) and +(left:7mm) .. (node4.west) node[midway, above] {$p_2$};
\draw[->] (node4.east) .. controls +(up:7mm) and +(left:7mm) .. (node5.west) node[midway, below] {$p_3$};
\end{tikzpicture}
}
\end{center}

\section{Theoretical Analysis}

\textbf{Proposition 1:} For any node $v$ in a computation graph, there exists at least one Arg Node which is a descendent of $v$ in the computation graph.

\begin{proof}
  Recall from Sec~\ref{sec:computation-graph} that in a computation graph, the sink nodes are the inputs and parameters of the model. To compute the gradients of the loss w.r.t. some parameter at layer $l$, as is the goal of backpropagation, we require the values of the parameters at layer $l + 1$ to perform chain rule, assuming the parameters are for linear projection matrices.
  
  Therefore, as long as the model inputs are set to track gradients, every sink node of the DAG will contain a parameter or input, which are forward activations.
  
  If any node did not eventually lead to such a sink node, that would imply that none of the nodes along this path require gradient tracking. But then this path should not exist in the computation graph at all, as it would have no impact on gradient computation.

  So, all other nodes in the computation graph must be ancestors of some such sink node, and this statement holds true.
\end{proof}

\subsection{DynamicLRP Overhead Complexity}
We analyse the computational complexity of our promise-based LRP algorithm in terms of the architectural properties that determine promise resolution requirements.

Let $V_P \subseteq V$ be the set of promise-generating operations in the computation graph $G = (V, E)$.
At the time of writing, the set of promise-generating operations is:
$$V_P = \{v \in V : \text{type}(f_v) \in \{\text{Add}, \text{Sum}, \text{Cat}, \text{Unbind}, \text{Mean}, \text{Stack}\}\}$$

\begin{definition}
For each promise-generating operation $v_p \in V_P$, the Promise Depth is:
$$d(v_p) = \max_{u \in outadj(v_p)} dist(u, argnode)$$
where for every $u$, $argnode$ is its nearest descendent operation that stores a forward activation required by $v_p$.
\end{definition}
\vspace{0.3cm}
\begin{theorem}[Promise-Based LRP Complexity]
Let $G = (V, E)$ be the computation graph of a neural network with $n = |V|$ operations, $m = |E|$ edges, promise-generating set $V_P$, and maximum promise depth $D$.

Let $C_{fwd}, C_{bwd}$ be the most expensive forward and backward pass computation steps, respectively.

Let $S$ be the size of the largest activation cached in the computation graph.

Let $\delta = \sum_{v_p \in V_P} d(v_p)$

Then, the promise-based LRP algorithm has runtime complexity $O((C_{fwd} + C_{bwd}) \cdot (n + m))$, and memory overhead complexity $O(|V_P| \cdot S)$.
\end{theorem}
\label{sec:lrp-complexity}
\begin{proof}
First, we consider that the memory bound is straightforward, as each Promise stores a constant number of tensors representing the \texttt{rout}, \texttt{rins}, and \texttt{args}, and all tensors are bounded above by the largest activation in the model computation.

To prove the time complexity, we analyze the algorithm in three distinct phases:

\textbf{Phase 1 - Initial Forward Pass:} One forward pass is performed to generate the model output and construct the computation graph. This requires $O(C_{fwd} \cdot (n + m))$ time.

\textbf{Phase 2 - Auxiliary Graph Construction:} The algorithm constructs the auxiliary graph $G'$ from the output's computation graph $G$, requiring only graph traversal and no numerical computation. This requires $O(n + m)$ time.

\textbf{Phase 3a - Backward LRP Pass:} The algorithm performs exactly one backward pass through $G'$ to compute relevance propagation. This requires $O(C_{bwd} \cdot (n + m))$ time.

\textbf{Phase 3b - Promise Resolution:} When the LRP traversal encounters a promise-generating operation $v_p \in V_P$ that lacks its required forward activation, a promise object is created. Each promise must traverse backward through the graph until it locates the operation storing its required activation. 

Crucially, our implementation shows that promises never share the same internal nodes in their computation paths. Each promise maintains independent traversal paths (via distinct \texttt{fwd\_list}, \texttt{bwd\_list}, and \texttt{arg\_node\_ind} fields) and resolves separately through individual \texttt{setarg()} calls.
Therefore, each promise requires exactly $d(v_p)$ traversal steps, leading to $d(v_p) \cdot C_{fwd}$ additional forward computations in the worst case, for a total of $\delta = \sum_{v_p \in V_P} d(v_p)$ forward computations when considering all Promises.
Since all computation paths are mutually exclusive, we have that $\delta \leq n$, and so promise resolution overhead has time complexity $O(C_{fwd} \cdot \delta) \in O(C_{fwd} \cdot n)$.

Therefore, the total complexity is in $O((C_{fwd} + C_{bwd}) \cdot (n + m))$, the asymptotic class of a standard backward pass, and memory overhead complexity $O(\delta \cdot S)$, where both $\delta$ and $S$ depend on architectural choices, and $\delta$ is independent of input size.
\end{proof}



\section{Input Aggregation}
When node $v$ receives multiple relevance inputs from in-neighbors, we aggregate before propagation:

\textbf{Tensor inputs} are typically summed element-wise: $R_v = \sum_{i=1}^k R_i$, but there are rare edge cases like when
propagating through a Split operation, where $R_v = concat(R_1 \ldots R_k)$, for example.

\textbf{Promise Branch inputs} $p_1, \ldots, p_k$ are merged into a single aggregation branch $p_{\text{agg}}$ with parent
connections to all $p_i$ but no child connections back. During forward-chaining, $p_{\text{agg}}$ distributes retrieved
activations to all parents. During backward-chaining, $p_{\text{agg}}$ sums relevances from all parents before continuing
propagation.

Notably, this kind of aggregation occurs naturally through the Pre-Promise mechanism. Consider that any such $v$ that
eventually receives more than one Promise Branch input, will at some point be in the exact state for requiring a
Pre-Promise ($|I| < \text{indegree}(v)$). This occurs right after $v$ receives its first Promise Branch input $p_1$.
As additional Promise Branches $p_2, \ldots, p_k$ arrive at $v$, they connect as parents to the Pre-Promise, and the
child connections are made when the Pre-Promise is promoted, completing the aggregation structure.

\textbf{Mixed Promise and Tensor inputs:} When $v$ receives both tensor relevances and Promise Branches, we handle Promise
aggregation as above. If only one Promise Branch is present, we still create a new aggregation branch $p_{\text{agg}}$ as
a child of the incoming branch. While this adds memory overhead, it maintains the Promise path abstraction required for
caching (\ref{sec:promise-caching}). If tensor relevances were allowed to merge directly into a Promise Branch's
internal nodes, cached Promise paths would behave incorrectly; relevance would be injected mid-chain rather than at the
Promise's Origin Node.

This process helps ensure that relevance is propagated only once through each node to prevent exponential blowup in graph
traversal and incorrect accumulation of relevance from multiple Promise Branches. This will help us prove Proposition 2:

\textbf{Proposition 2:} Each node in the computation graph has relevance propagated through it exactly once during LRP traversal.

\begin{proof}
  We prove this in two cases: standard propagation and Promise propagation.

  \textbf{Case 1: Standard Propagation.}
  By our traversal heuristic, and facilitated by the Input Aggregation process, we have that this statement will hold true in any case that
  follows the heuristic.

  By traversal heuristic, we process node $v$ only when all in-neighbors have propagated relevance to it. Since the
  computation graph is a DAG (no operation can depend on its own output), no traversal path starting from $v$ can return
  to $v$. 
  Input aggregation ensures that multiple incoming relevances are combined before processing, so $v$ is processed exactly
  once when all inputs arrive.

  \textbf{Case 2: Promise Propagation.}
  The two distinct phases of a Promise Branch are the forward chaining of activations when it reaches an Arg Node and the
  backward chaining of relevance, deferred until all parents in the Promise Tree are in Complete state.
  Therefore, even though the nodes within a Promise Branch path are traversed in the first phase, they do not have
  relevance propagated through them until the backward chain is executed, propagating all the way through to the Arg Node
  and skipping all internal nodes. Since the Arg Node is a descendent of $v$ and all internal nodes in the chain, by the
  DAG argument, none of those nodes can be revisited.

\end{proof}

\textbf{Claim:} The internal nodes of any two distinct Promise branches are disjoint.
\label{appendix:unique-chains}

\begin{proof}
  By Proposition 2 we have that no node has relevance propagated through it more than once in graph traversal.

  Suppose for contradiction that two distinct Promise branches $p_1$ and $p_2$ share an internal node $v$.
  However, this means that during traversal, both $p_1$ and $p_2$ had to traverse $v$, or else it would not be recorded in
  their chains. In that case, we would have applied Case 2 of Input Aggregation via the Pre-Promise mechanism, and both
  chains would then terminate at the creation of the Pre-Promise. But, this contradicts our premise that $v$ was an internal
  node for both branches, therefore it must be that no two distinct Promise branches share any internal nodes.
\end{proof}

\textbf{Remark.} This property ensures that the total number of nodes traversed by 
all Promises is bounded by $\delta = \sum_{v_p \in V_P} d(v_p) \leq |V|$, as stated 
in Theorem 1 (\ref{sec:lrp-complexity}). Each node contributes to at most one Promise's internal chain, preventing 
double-counting in complexity analysis.

\section{List of Covered Operations}
\label{appendix:covered-operations}
Below is the list of currently supported operations for DynamicLRP.
\begin{table}[!h]
\begin{center}
  \begin{tabular}{ll}
    \toprule
    \textbf{Category} & \textbf{Operations} \\
    \midrule
    \textbf{Arithmetic} & Add, Sub, Mul, Div, Neg, Sqrt, Rsqrt, Pow, Exp, Log \\
    \textbf{Linear Algebra} & Mm, Bmm, Convolution \\
    \textbf{Aggregation} & Sum, Mean, Cat, Stack, Unbind, Split, Max, Gather \\
    \textbf{Shape Manipulation} & View, Reshape, Transpose, Permute, Expand, Repeat, Squeeze, Unsqueeze \\
    \textbf{Indexing and Selection} & Slice, Index, Select \\
    \textbf{Pooling} & MaxPool2D, MaxPool3D, AdaptiveAvgPool2D \\
    \textbf{Normalization} & NativeLayerNorm, NativeBatchNorm \\
    \textbf{Attention} & SDPA \\
    \textbf{Activations} & GELU, SiLU, ReLU, Softmax \\
    \textbf{Other} & MaskedFill, Clone, ToCopy, NativeDropout, Embedding \\
    \bottomrule
  \end{tabular}
\end{center}
\end{table}

\section{Extra Results}
\label{sec:appendix_extra_results}
\begin{table}[!h]
\begin{center}
\scriptsize
  \caption{ABPC and attribution efficiency metrics for VGG experiments. 3800 examples from ImageNette (320px) validation set.}
  \label{vgg-table}
  \begin{tabular}{ lccccc }
    \toprule
    \textbf{Method} & \textbf{ABPC ($\uparrow$)} & \textbf{Comprehensiveness ($\uparrow$)} & \textbf{Sufficiency ($\downarrow$)} & \textbf{Speed (it/s, $\uparrow$)} & \textbf{Peak VRAM ($\downarrow$)} \\
    \midrule
    Saliency & 0.73 & 2.58 & 1.85 & \textbf{156.83} & \textbf{0.89GB} \\ 
    SmoothGrad & 1.50 & \textbf{2.87} & 1.37 & 8.90 & 0.90GB \\ 
    Random & -0.0022 & 2.40 & 2.40 & - & - \\
    Zennit LRP & 1.69 & 2.72 & 1.02 & 29.41 & 1.4GB \\
    \midrule
    Ours & \textbf{1.77} & 2.65 & \textbf{0.88} & 21.79 & 2.3GB \\
    \bottomrule
  \end{tabular}
\end{center}
\end{table}
\begin{table}[!tph]
\scriptsize
\begin{center}
  \caption{ABPC and attribution efficiency metrics for ViT-b-16 experiments. 1000 examples from CIFAR-10 test set.}
  \label{vit-table}
  \begin{tabular}{ lccccc }
    \toprule
    \textbf{Method} & \textbf{ABPC ($\uparrow$)} & \textbf{Comprehensiveness ($\uparrow$)} & \textbf{Sufficiency ($\downarrow$)} & \textbf{Speed (it/s, $\uparrow$)} & \textbf{Peak VRAM ($\downarrow$)} \\
    \midrule
    IG & 0.10 & 0.95 & 0.85 & N/A* & >12GB \\ 
    InputXGrad & -0.0030 & 0.90 & 0.90 & \textbf{59.89} & 2.1GB \\ 
    GradSHAP & 0.074 & 0.94 & 0.87 & 3.32 & 8.6GB \\
    Random & -0.023 & 0.88 & 0.90 & - & - \\
    AttnLRP & \textbf{1.46} & \textbf{1.78} & 0.33 & 6.02 & \textbf{0.96GB} \\
    \midrule
    Ours & \textbf{1.46} & \textbf{1.78} & \textbf{0.32} & 9.74 & 1.6GB \\
    \bottomrule
  \end{tabular}
\end{center}
\vspace{0.1cm}
*N/A in the Speed column signifies that GPU memory limits caused thrashing and prevented the observation of method's true speed.
\vspace{10pt}
\end{table}
\begin{figure}[!h]
  \centering
  \includegraphics[width=1.0\textwidth]{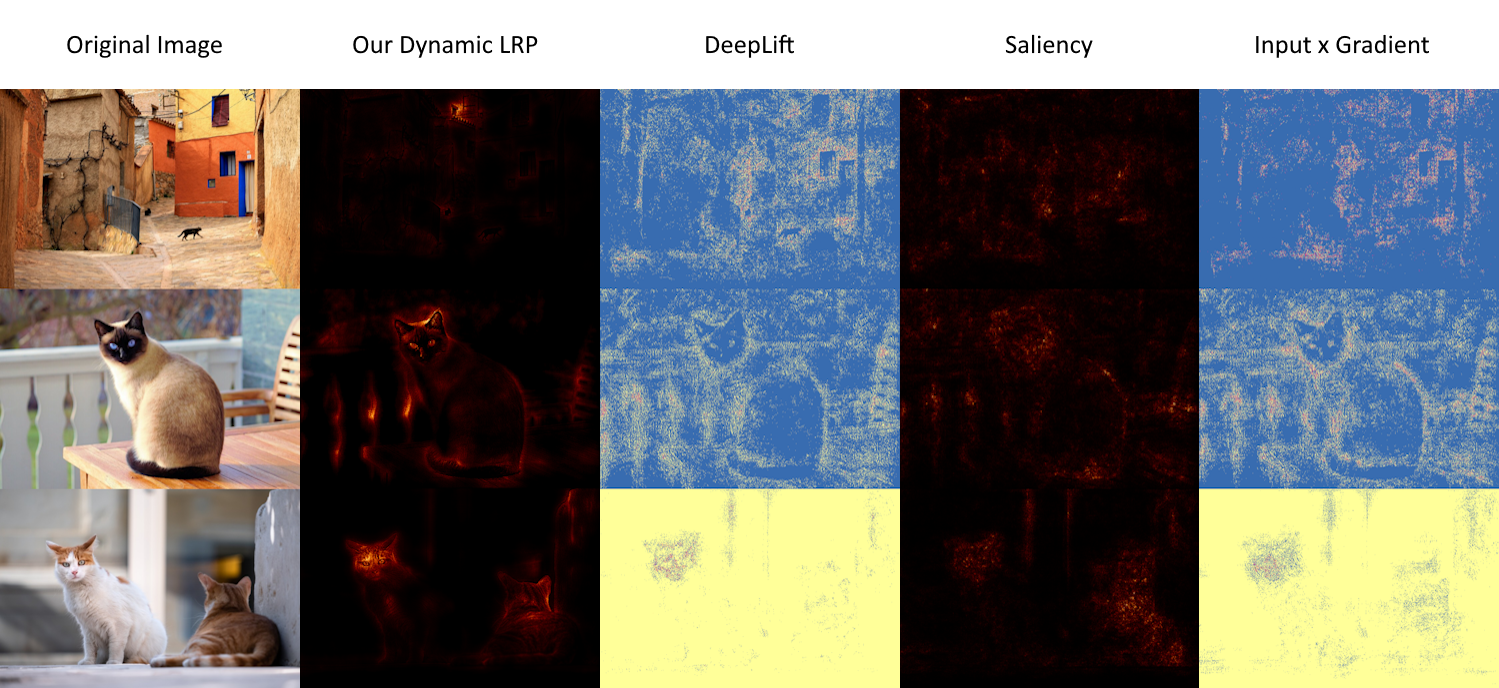}
  \caption{Comparing methods for VGG attributions.}
\end{figure}
\begin{table}[!bh]
\scriptsize
\begin{center}
  \caption{DynamicLRP Attribution faithfulness measured by top attributed token accuracy against model predictions and labels, and IoU with predictions with the entire SQuADv2 validation set (5928 examples). EM = Exact Match. Higher is better for all but the last two columns.
  Columns marked with (P) are w.r.t. model predictions, (L) are w.r.t. labels. Top-2 accuracy and IoU values in parentheses signify strict
  matching to only the predicted start/end indices, and none of the indices between them. Non-parenthesised values here allow hits on indices
  within the span. All remaining unaccounted examples in the totals were skipped due to exceeding the model max token length (512).}
  \label{squadv2-table}
  \begin{tabular}{ lcccccccc }
    \toprule
    \textbf{Model} & \textbf{EM} & \textbf{F1} & \textbf{top-1 (L)} & \textbf{top-1 (P)} & \textbf{top-2 (P)} & \textbf{IoU (P)} & \textbf{Examples} & \textbf{Skipped} \\
    \midrule
    RoBERTa-L & 82.29 & 86.27 & 93.70 & 97.47 & 96.22 (88.95) & 92.70 (80.09) & 5844 & 45 \\
    Flan-T5-L & 83.65 & 90.15 & 95.06 & 98.64 & 94.73 (85.64) & 89.94 (74.86) & 5830 & 36 \\
    \bottomrule
  \end{tabular}
\end{center}
\end{table}
\begin{figure}[!ht]
    \centering
    \begin{subfigure}[b]{0.8\textwidth}
        \centering
        \includegraphics[width=1.0\linewidth]{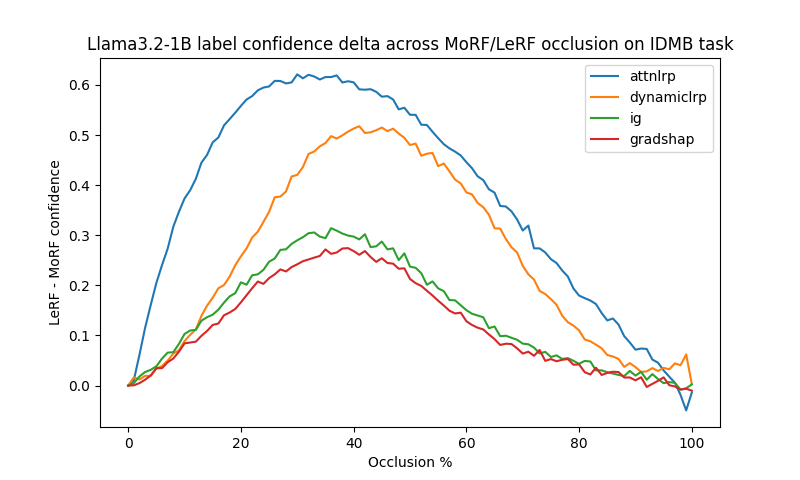}
        \caption{}
    \end{subfigure}
    \begin{subfigure}[b]{0.8\textwidth}
        \centering
        \includegraphics[width=1.0\linewidth]{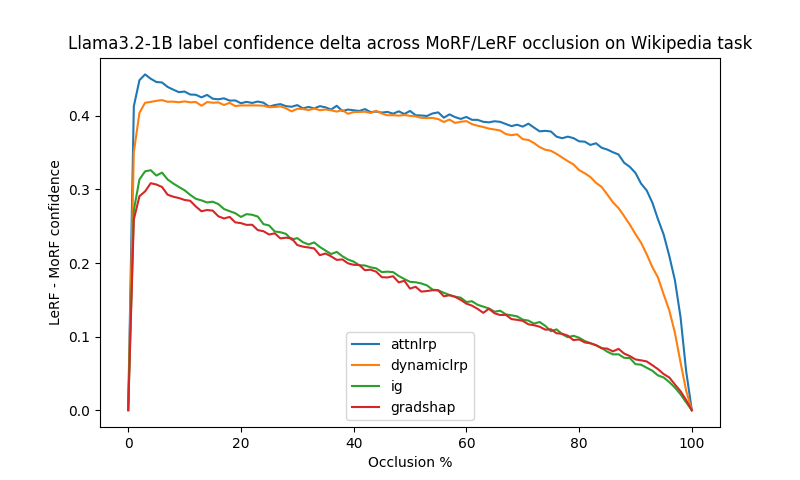}
        \caption{}
    \end{subfigure}
    \hfill
        \caption{Plotting the LeRF - MoRF difference in label confidence of Llama3.2-1B on IMDB classification in \textit{(a)} and Wikipedia next-token prediction in \textit{(b)}. A higher curve is better, signifying bigger gaps in prediction confidence when removing the features with highest vs. lowest importance.}
\label{fig:llama-plots}
\end{figure}
\begin{figure}[h]
  \centering
  \includegraphics[width=1.0\textwidth]{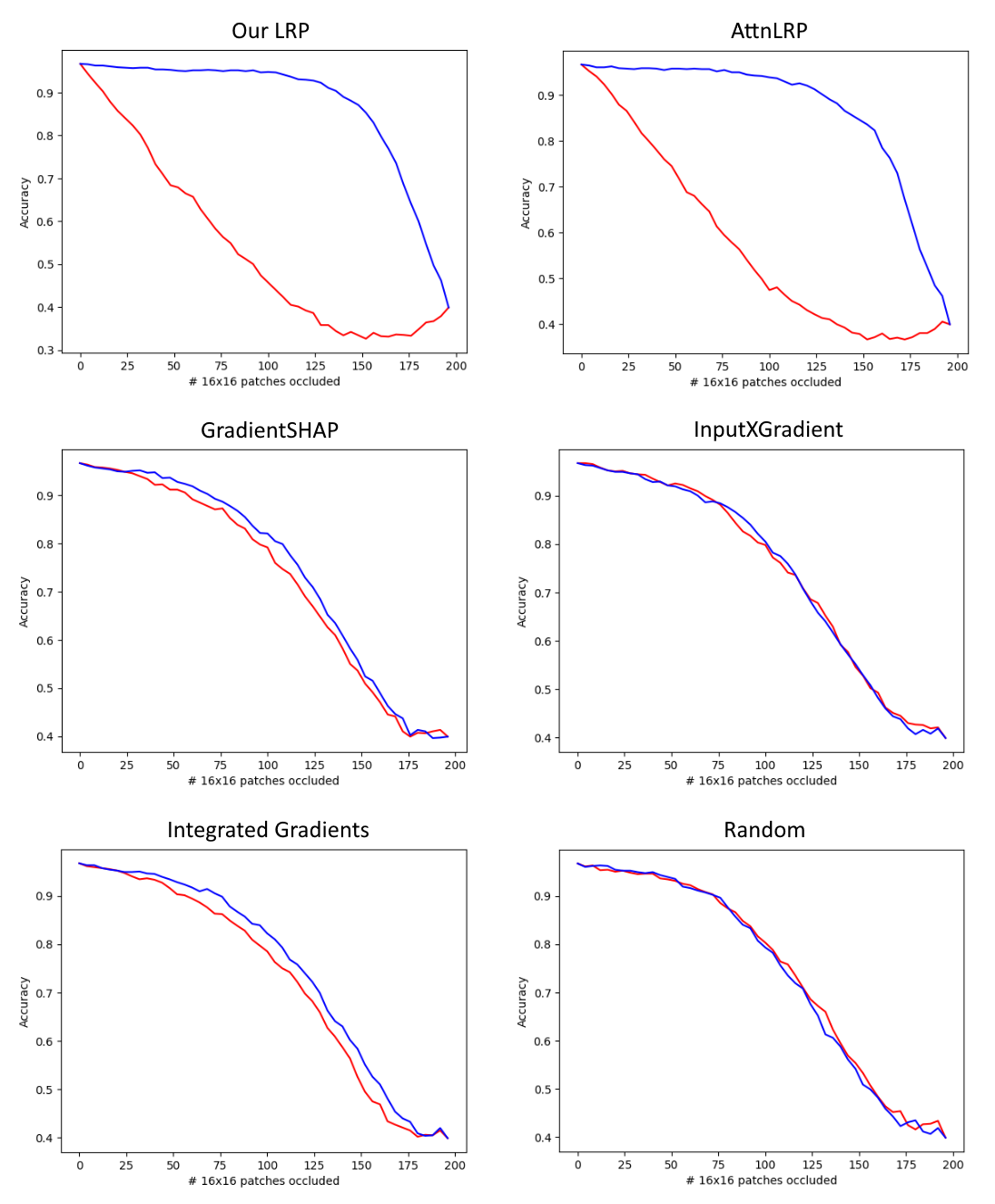}
  \caption{Comparing faithfulness of methods using ABPC from pretrained ViT-b-16 performance on 1000 examples from the ImageNet CIFAR-10 test set. Occlusion is applied by replacing 16x16 patches with the corresponding regions from a Gaussian-blurred version of the image (kernel size = 51, $\sigma$ = 20). Red = MoRF curve, Blue = LeRF curve.}
\end{figure}
\begin{table}[!h]
\scriptsize
\begin{center}
  \caption{DynamicLRP operation coverage breakdown of 15 models spanning vision, text, and audio modalities. \textbf{Unique Ops} indicates the number of distinct operation types in each model's computation graph.}
  \label{coverage-table}
  \begin{tabular}{ lcccc }
    \toprule
    \textbf{Model} & \textbf{Modality} & \textbf{Covered Nodes} & \textbf{Unique Ops covered} \\
    \midrule
    VGG16 \cite{simonyan2015vgg} & Vision & 90/90 & 9/9 \\
    ResNet-50 \cite{he2015resnet} & Vision & 339/339 & 10/10 \\
    ViT-b-16 \cite{dosovitskiy2021vit} & Vision & 779/779 & 16/16 \\
    EfficientNetv2-medium \cite{tan2021efficientnetv2} & Vision & 1,526/1,526 & 12/12 \\
    SigLIP-2-So400m-base14-384 \cite{tschannen2025siglip2} & Vision & 2,178/2,178 & 19/19 \\
    Wav2Vec2-xls-r-300m \cite{babu2021xlsr} & Audio & 2,021/2,022 & 18/19 \\
    whisper-small \cite{radford2022whisper} & Audio & 2,729/2,729 & 22/22 \\
    GPT-2 \cite{radford2019gpt2} & Language & 885/885 & 19/19 \\
    RoBERTa-large \cite{liu2019roberta} & Language & 1,461/1,461 & 14/14 \\
    Llama3.2-1B \cite{dubey2024llama3} & Language & 1,787/1,787 & 24/24 \\
    Gemma-3-270m-it \cite{gemmateam2025gemma3} & Language & 2,482/2,482 & 24/24 \\
    Qwen3-0.6B \cite{yang2025qwen3} & Language & 2,590/2,590 & 18/18 \\
    Flan-T5-large \cite{chung2022flant5} & Language & 5,713/5,713 & 24/24 \\
    DePlot \cite{liu2022deplot} & Vision/Language & 2,863/2,863 & 23/23 \\
    Mamba-130m \cite{gu2024mamba} & State Space & 3,421/3,421 & 26/26 \\
    \midrule
    \textbf{Aggregate} & - & \textbf{31,464/31,465} & - \\
    \bottomrule
  \end{tabular}
  \vspace{5pt}
  \caption{Uncovered operation breakdown, where applicable.}
  \begin{tabular}{| c | c  c |}
    \hline
    {Model} & {Uncovered Ops} & {Node count} \\
    \hline
    Wav2Vec2-xls-r-300m & WeightNormInterface & 1 \\
    \hline
  \end{tabular}
\end{center}
\end{table}


\begin{table}[!h]
\begin{center}
\scriptsize
  \caption{Parameter randomization sanity check, as outlined in \citet{adebayo2020sanitycheckssaliencymaps}, for DynamicLRP ViT-B-16 attributions on 100 CIFAR10 image classification examples. Layer randomization is done from output to input layer. For each sample, the Spearman correlation ($\rho$) was taken between the attribution produced by the trained model $M_0$ and each randomized model variant $M_i$, then the average over all 100 samples was taken to evaluate attribution deterioration throughout the randomization process.}
  \label{param-randomization}
  \begin{tabular}{ lc }
    \toprule
    \textbf{Layers Randomized ($i$)} & \textbf{Mean $\rho_{0,i}$} \\
    \midrule
    1 & 0.479 \\ 
    2 & 0.297 \\ 
    3 & 0.127 \\ 
    4 & 0.0239 \\ 
    5 & 0.0391 \\ 
    6 & 0.0310 \\ 
    7 & 0.0414 \\ 
    8 & 0.0214 \\ 
    9 & 0.0204 \\ 
    10 & 0.0260 \\ 
    11 & 0.00678 \\ 
    12 & 0.00785 \\ 
    
    \bottomrule
  \end{tabular}
\end{center}
\end{table}

\end{document}